\documentclass[11pt]{article}
\usepackage{framed,epsf,latexsym,amsthm,amsmath,amssymb,amscd,textcomp}
\usepackage[numbers]{natbib}
\usepackage{graphicx}
\usepackage{pgfplots}
\usepackage{tikz}
\usepackage{algorithm}
\usepackage{algorithmic}

\usepackage[margin=1.0in]{geometry}

\newtheorem{theorem}{Theorem}[section]
\newtheorem{lemma}[theorem]{Lemma}
\newtheorem{remark}[theorem]{Remark}

\newtheorem{counter-example}[theorem]{Counter example}

\newtheorem{open question}[theorem]{Open question}
\newtheorem{corollary}[theorem]{Corollary}

\newtheorem{definition}[theorem]{Definition}

\newcommand{\proofbox}{\hfill $\Box$}

\newcommand{\cM}{{\cal M}}
\DeclareMathOperator{\poly}{poly}

\newcommand{\ignore}[1]{}

\newcommand{\ca}{{\cal A}}
\newcommand{\cb}{{\cal B}}

\newcommand{\cl}{{\cal L}}

\newcommand{\cx}{{\cal X}}

\newcommand{\valpha}{{\vec{\alpha}}}

\newcommand{\ymerf}{{\mathrm erf}}

\newcommand{\opt}{\mathrm{OPT}}

\newcommand{\base}{\mathrm{base}}

\newcommand{\OLO}{\mathrm{OLO}(N,D)}
\newcommand{\EXP}{\mathrm{EXP}(N,D)}

\newcommand{\reals}{{\mathbb R}}

\newcommand{\val}{\mathrm{VAL}}

\newcommand{\ymE}{\mathbb{E}}

\newcommand{\inner}[1]{\langle #1 \rangle}

\title{Competitive ratio versus regret minimization: achieving the best of both worlds}

\author{
    Amit Daniely\thanks{The Hebrew University of Jerusalem and  Google Research, Tel Aviv} \hspace{1cm}
    Yishay Mansour\thanks{Blavatnik School of Computer Science, Tel Aviv University, Tel Aviv Israel
       and  Google Research, Tel Aviv}
}

\usepackage{times}

\begin{document}

\maketitle

\begin{abstract}
We consider online algorithms under both the competitive ratio
criteria and the regret minimization one. Our main goal is to build
a unified methodology that would be able to guarantee both criteria
simultaneously.
For a general class of online algorithms, namely any Metrical Task
System (MTS), we show that one can simultaneously guarantee the best
known competitive ratio and a natural regret bound. For the paging
problem we further show an efficient online algorithm (polynomial in the number of pages) with this guarantee.

To this end, we extend an existing regret minimization algorithm
(specifically,
\cite{kapralov2011prediction}) to handle movement cost (the cost of
switching between states of the online system). We then show how to
use the extended regret minimization algorithm to combine multiple
online algorithms. Our end result is an online algorithm that can
combine a ``base" online algorithm, having a guaranteed competitive
ratio, with a range of online algorithms that guarantee a small
regret over any interval of time. The combined algorithm guarantees
both that the competitive ratio matches that of the base algorithm
and a low regret over any time interval.

As a by product, we obtain an expert algorithm with close to optimal
regret bound on every time interval, even in the presence of
switching costs. This result is of independent interest.
\end{abstract}

\newpage

\section{Introduction}


Online algorithms address decision making under uncertainty. They serve a
sequence of requests while having uncertainty regarding future requests.
We consider the Metrical Tasks Systems (MTS) framework for analyzing online algorithms.
In this framework, the online algorithm first receives a request and then decides how to serve it.
In order to serve the request, there are two
types of costs. A cost for changing the state of the underlying system, and a cost for serving the request from the new state. The cost of
the online algorithm is the sum of these two costs.

%
%

Historically, the initial dominant form of analysis for such
algorithms was to assume a stochastic arrival process for the
requests, and analyze the performance of the online algorithm given
it (e.g., \cite{CoffmanL91}).
Given a well defined arrival process, there exists a well defined optimal policy and an associated optimal cost.
This methodology is quite sensitive to the modeling assumptions, from the specific arrival process, to
the assumption about dependencies (e.g., i.i.d. requests), to the robustness when the assumptions are not perfectly met.

Competitive analysis is aimed at addressing those issues and giving a worst case guarantee.
Rather than assuming a stochastic arrival process,
competitive analysis allows for any request sequence. The main idea
is to compare the performance of the online algorithm on the request
sequence to that of an omniscient offline algorithm that observes in
advance the entire request sequence. The worse case ratio, over all
possible request sequences, between the performance of the online
algorithm and the performance of the omniscient offline algorithm is
the competitive ratio. (See, e.g., \cite{BorodinEl98}.) While for
a few online tasks there are algorithms with good (i.e., small) competitive ratios, the competitive analysis approach is often criticized as being too
pessimistic. Indeed, for many online tasks there are no algorithms with good competitive ratio.

Another form of analyzing online algorithms is regret minimization.
In this approach, originally developed for analyzing online
prediction tasks, one fixes a collection of $N$ benchmark
algorithms. Then, the loss of the algorithm at hand is compared to
the loss of the best performing algorithm in the benchmark. More
specifically, the regret is the difference between the cumulative
loss of the online algorithm and that of the best benchmark
algorithm, and the goal is to have vanishing average regret.
Vanishing regret implies matching the performance of the best
benchmark algorithm. (See, e.g., \cite{CesaBianchiLu06}.)




Our main goal is to develop a methodology that would allow to keep
best known competitive ratio guarantees, while augmenting them with additional regret minimization guarantees. Specifically, given an online algorithm we will guarantee that our resulting algorithm would have the same competitive ratio of the original algorithm. In addition, we will guarantee that for any time interval we will have a low regret, when compared to the benchmark of algorithms that do not switch states.

Specifically, for any time
interval $I=[t_1,t_2]$ the regret is at most
$O(\sqrt{|I|\log (TN)})$, where $T$ is the total number of time
steps and $N$ is the number states.

This implies, for instance, that if the benchmark makes $s$ state changes, the
regret bound would be bounded by $O(\sqrt{sT\log(TN)})$, and for $s
\ll T$ the regret is sublinear.
To derive our technical results we introduce a new regret
minimization algorithm for the experts problem.
Here, the algorithm has to choose one out of $N$ experts at each round. Following its choice, a loss for each expert is revealed, and the algorithm suffers the loss of the expert it chose. The regret of the algorithm is its cumulative loss, minus the cumulative loss of the best expert.

Our algorithm is inspired by the algorithm
of \cite{kapralov2011prediction}. The main benefit of
\cite{kapralov2011prediction} is that they have essentially zero
(exponentially small constant) regret to one expert while having the
usual $O(\sqrt{T\log(N)})$ regret with respect to the other experts.
Our main technical contribution is to extend the algorithm to work
in a {\em strongly adaptive} and {\em switching costs} setting. By
strong adaptivity we mean that we can bound the regret of any time
interval $I$ as a function of the length of $I$, i.e., $|I|$, rather
than the total number of time steps $T$. By  switching cost we mean
that we associate a cost with changing experts between time
steps. The ability of handling switching costs is critical in order
to extend regret minimization results to the competitive analysis
framework.
%
Finally, as in \cite{kapralov2011prediction}, we have a base expert,
such that the regret with respect to it is essentially zero.
This base expert will model the online competitive algorithm whose
performance we like to match.
We also maintain an algorithm for each interval size (rounded to powers of two).

Using our regret minimization algorithm we can show that for any
metrical task system (MTS) there is an online algorithm that
guarantees the minimum between: (1) the competitive ratio times the
offline cost plus an additive constant $D$, the maximum switching cost, and
(2) for any time interval the regret with respect to any service sequence with $s$ state changes
is at most $O(D\sqrt{s|I|\log (NT)})$.
%
%
In addition, the computation time of the algorithm is polynomial in
$T$ and $N$. We note however that in many applications the number of states $N$
is exponential in the natural parameters of the problem. For instance, in the
paging problem, where the cache has $k$ out of $n$ pages, the number
of states is $N=\Omega(n^k)$. However, we can show how to run our online
algorithm in time polynomial in $n$ and $k$. For the $k$-server
problem we derive an algorithm which is unfortunately not polynomial
in $n$ and $k$. We leave the existence of an efficient algorithm as an open problem, but do point on some possible obstacles. Namely, we show that an efficient sublinear regret
algorithm would imply an improved approximation algorithm
for the well studied $k$-median problem.

{\bf Related Work.}
Adaptivity has attained much attention in the online learning
literature over the years (an incomplete list includes
~\cite{herbster1998tracking, bousquet2003tracking, BlumMa07,hazan2007adaptive, kapralov2011prediction,
cesa2012new, adamskiy2012closer,
panigrahy2013optimal, luo2015achieving, daniely2015strongly, altschuler2018online}. Another
related, but somewhat orthogonal line of
work~\cite{zinkevich2003online,hall2013online,rakhlin2013optimization,
jadbabaie2015online} studies {\em drifting environments}. Switching costs were studied in~\cite{KalaiVe05}
and attained considerable interest recently (e.g,
\cite{geulen2010regret, cesa2013online, dekel2014bandits}).
Connections between regret-minimization and competitiveness has been
previously studied. In particular \cite{blum2002static,
blum2000line, abernethy2010regularization, buchbinderunified} focus
on algorithmic techniques that simultaneously apply to MTS and
experts, and also show regret minimization algorithms for
MTS an other online problems. Another example
is~\cite{andrew2013tale} that studies tradeoffs between regret minimization
and competitiveness for a certain online convex problem.

Technically, the work of \cite{kapralov2011prediction} is quite
close to our work.
They showed an algorithm for the two experts problem that has
essentially no regret w.r.t. one of the experts, and nearly optimal
regret ($\tilde O(\sqrt{T})$) w.r.t. the second expert. They also
showed that this result implies an algorithm with strongly adaptive
guarantees in the $N$-experts problem. Namely, they showed an
algorithm with nearly optimal regret on every {\em geometric} time
interval.
Our main contribution, regarding regret minimization, is to handle
switching costs. In addition, we extend the regret bound
to hold on all intervals.

\section{Model and Results}\label{sec:results}

\subsection{Strongly adaptive expert algorithms in the presence of switching costs}

We consider an online setting where there are $N$ experts and a
switching cost of $D\geq 0$, which we call $N$-experts $D$-switching
cost problem. There are $T$ time steps, and at time $t = 1,\ldots,
T$, the online algorithm chooses an expert $i_t\in [1,N]$. Then, the
adversary reveals a loss vector $l_t=(l_t(1),\ldots l_t(N))\in
[0,1]^N$. The loss of the algorithm at time $t$ is $l_t(i_t)$. In
addition, if $i_t$ is different from $i_{t-1}$, the algorithm
suffers an additional switching cost of $D$. This implies that the
total loss of the algorithm is $\sum_{t=1}^Tl_t(i_t) + D\sum_{t=2}^T1[i_t\ne i_{t-1}]$.

We assume that the adversary is oblivious, i.e., the sequence of
losses is chosen before the first time step. Likewise, we assume
that the time horizon (denoted by $T$) is known in advance.
Given a time interval $I = \{t_0+1,...,t_0+k\}$ we say that the
online algorithm has a regret bound of $R(I)$ in $I$ if for any
sequence of losses we have
\begin{equation}\label{eq:regret}
\ymE\left[\sum_{t\in I} l_t(i_t) + D\sum_{t\in I \setminus
\{t_0+1\}} 1[i_t\ne i_{t-1}]\right] \le \min_{i\in [1,N]}\sum_{t\in
I} l_t(i) + R(I)\;,
\end{equation}
where $1[\cdot]$ is the indicator function and the expectation is
over the randomization of the online algorithm. Note that we sum the
losses only for $t\in I$ and the switching costs only for $t\in
\{t_0+2,...,t_0+k\}$. When we refer to the {\em regret of the
algorithm in $I$} we mean the minimal $R(I)$ for which inequality
\eqref{eq:regret} holds. We say that the algorithm has a regret
bound of $R(T)$ if it has a regret bound of $R(T)$ in the interval
$[1,T]$.
Our first result presents an algorithm whose regret on any interval
$I$ is close to optimal, even when switching costs are present.
Specifically, we show that
\begin{theorem}\label{thm:main_non_uniform}
There is an $N$-experts $D$-switching cost algorithm with
$O\left(N\log(T)\right)$ per-round computation, whose regret on
every interval $I\subseteq [T]$ is at most $O\left(\sqrt{(D+1)
|I|\log(NT)}\right)$
\end{theorem}

We next show that Theorem \ref{thm:main_non_uniform} is tight, up to
the dependency on $D$.
Recall that even when ignoring switching costs
and considering only interval $I$ the regret is
$\Omega\left(\sqrt{|I|\log(N)}\right)$ (e.g.,
~\cite{CesabianchiFrHeHaScWa97}).
The following theorem improves the lower bound by showing that
in order to have a regret bound for {\em any} interval there is an
additional log factor of $T$ (even for $D= 0$).
\begin{theorem}\label{thm:main_lower}
For every $N$-experts algorithm there is segment $I$ with regret  $\Omega\left(\sqrt{|I|\log\left(NT\right)}\right)$.
\end{theorem}

\subsection{Metrical Task systems and competitive analysis}


{\bf MTS model:} An MTS is a pair $(\cx,\cl)$ where $\cx$ is an
$N$-points (pseudo-)metric space and $\cl\subset [0,1]^\cx$ is a
collection of possible loss vectors. Each  $x\in\cx$ represents a
state of the MTS, and the distance function $d(x_1,x_2)$ describes
the cost of moving between states $x_1,x_2\in \cx$.
At each step $t=1,\ldots,T$ the online algorithm is first given a
loss vector $l_t\in \cl$ and then chooses a state $i_t\in \cx$.
Following that, the algorithm suffers a loss of $l_t(i_t) +
d(i_t,i_{t-1})$. Therefore, for a sequence of $T$ steps the toal loss would be
$\sum_{t=1}^T l_t(i_t) + d(i_t,i_{t-1})$. As in the experts problem, we assume that the adversary is oblivious and that the time horizon (denoted by $T$) is known in advance.

{\bf MTS versus Experts:}
There are three differences between the MTS model and the
experts model:
(i) In the MTS problem the algorithm observes the loss before it
chooses an action while in the expert problem the algorithm first selects
the action and only then observes the losses.
(ii) In MTS the switching costs are dictated by an arbitrary metric,
while in our experts setting the switching cost is always the same.
(iii) In the expert problem the losses can be arbitrary, while in
MTS the losses are restricted to be elements of $\cl$.

{\bf Competitive Analysis:}
One of the classical measures of online algorithms is {\em
competitive ratio}.
%
%
An online algorithm has a competitive ratio $\alpha\ge 1$ if, up to
an additive constant and for any  sequence of losses, the loss of the algorithm is bounded by
$\alpha$ times the loss of the best possible sequence of actions. In other worlds, the loss of the
online algorithm is competitive with the optimal offline algorithm that is
given the sequence of losses (requests) in advance.
Formally, there is a constant $\beta>0$ (independent of $T$) such
that for any sequence of losses $l_1,\ldots, l_T$ we have that
\begin{equation}\label{eq:competitive}
\ymE\left[\sum_{t=1}^Tl_t(i_t) + \sum_{t=2}^Td(i_t, i_{t-1})\right]
\le
\alpha\min_{i^*_1,\ldots,i^*_T\in\cx}\left(\sum_{t=1}^Tl_t(i^*_t) +
\sum_{t=2}^Td(i^*_t, i^*_{t-1})\right)+ \beta\;.
\end{equation}
Classical results by~\cite{borodin1992optimal} provide, for any
MTS, a deterministic algorithm with competitive ratio $2N-1$. This
is optimal, in the sense that there are MTSs with no deterministic
algorithm with a better competitive ratio. For randomized
algorithms, the best known competitive ratio~\cite{fiat2003better,
bartal1997polylog} for a general MTS is
$O\left(\log^2(N)\log\log(N)\right)$, while the best known lower
bound~\cite{bartal2005metric, bartal2006ramsey} is
$\Omega\left(\frac{\log(N)}{\log\log(N)}\right)$.

{\bf Regret for MTS:}  The notions of regret and regret on a time
interval for MTS are defined in a fashion similar to the experts
problem. Given time interval $I = \{t_0+1,...,t_0+k\}$ we say that
the algorithm has a regret bound of $R(I)$ in $I$ if for any
sequence of losses we have
\begin{equation}\label{eq:regret_mts}
\ymE\left[\sum_{t\in I} l_t(i_t)  + \sum_{t\in I \setminus
\{t_0+1\}} d(i_t, i_{t-1})\right] \le \min_{i\in [1,N]}\sum_{t\in
I}l_t(i) + R(I)\;,
\end{equation}
where the expectation is over the randomization of the algorithm. Likewise, we say that the algorithm has a regret bound of $R(T)$ if it has a
regret bound of $R(T)$ in the interval $[1,T]$.

{\bf Competitive ratio versus regret minimization:}
Both competitive analysis and regret minimization are measuring the
quality of an online algorithm compared to an offline counterpart.
There are two major differences between the two approaches. The
first is the type of benchmark used for the comparison: Competitive
analysis allows an arbitrary offline algorithm while regret
minimization is limiting the benchmark to the best static expert,
i.e., selecting the same expert at every time step. The second
difference is the quantitative comparison criteria. While in
competitive analysis the comparison is multiplicative, in the regret
analysis the comparison is additive.

From Theorem \ref{thm:main_non_uniform} with switching costs of $D=\max_{x,y\in X}d(x,y)$ it is not hard to conclude that:
\begin{corollary}\label{cor:MTS_no_competitie}
For any MTS there is an algorithm whose regret on any interval
$I\subseteq [1,T]$ is $O\left(\sqrt{(D+1)|I|\log\left(NT\right)}
\right)$.
\end{corollary}

Based on a variant of Theorem \ref{thm:main_non_uniform} (namely,
Theorem \ref{thm:main_combine} below) we show that,
\begin{corollary}\label{cor:MTS_competitie}
Given an online algorithm $\ca_\base$ for a MTS $\cM$ with
competitive ratio $\alpha\ge 1$, there is an online algorithm $\ca$
for $\cM$  such that
\begin{itemize}
\item $\ca$ has a competitive ratio of $\alpha$
\item The regret of $\ca$ on every interval $I\subseteq [T]$ is
$O\left(D\sqrt{|I|\log\left(T\right)} +
\sqrt{(D+1)|I|\log\left(NT\right)} \right)$.
\end{itemize}
Furthermore, the per-round computational overhead of $\ca$ on top of
$\ca_\base$ is $O\left(N\log(T)\right)$.
\end{corollary}


\subsection{Paging and $k$-sever}

The result regarding the MTS framework (Corollary
\ref{cor:MTS_competitie}) gives a general methodology to achieve the best of both world:
guaranteeing a low regret with respect to the best static solution
and at the same time guaranteeing a good competitive ratio.
One drawback of Corollary \ref{cor:MTS_competitie} is the
dependency on the number of states, $N$. While for an abstract
setting, such as MTS, one should expect at least a linear dependency
on the number of states $N$, in many concrete application this
number is exponential in the natural parameters of the problem.
In this section we discuss two such cases, the paging problem and
the $k$-server problem. For the paging problem we show how to
overcome the computation issue. For the $k$-server problem we show some
possible computational limitations.


{\bf Paging.}
In the online paging problem there is a set $P$ of $n$ memory pages, out of which $k$ can be in the cache at a given time.
At each time we have a request for a page, and if the page is not
located in the cache we have a {\em cache miss}. In this case,
the algorithm has to fetch the page from memory to the cache
and incur a unit cost. If when it fetches the page, the cache is
full (has $k$ pages) then it also has to evict a page.
The algorithm tries to minimize the number cache misses.


It is fairly straightforward to model the paging problem as an MTS.
The states of the MTS will be all possible configuration in which
the cache is full, i.e., all $C\subset P$ such that $|C|=k$. (We can
assume without loss of generality that the cache is always full.)
The number of states of the MTS is $N=\binom{n}{k}$. We need to
define a metric between the states and possible loss functions.
Given two configurations $C_1$ and $C_2$ let
$d(C_1,C_2)=|C_1\setminus C_2|$. First, note that since the cache is
always full, the distance is symmetric. Second, moving from cache
$C_1$ to cache $C_2$ involves fetching the pages $C_1\setminus C_2$
and evicting the pages $C_2\setminus C_1$, and has cost
$|C_1\setminus C_2|$. Now we need to define the possible loss
function. We have a possible loss function $\ell_i$ for each request
page $i\in P$. For a cache $C$, if $i\in C$ then $\ell_i(C)=0$ and
if $i\not\in C$ we have $\ell_i(C)=2$. Clearly when $i\in C$ we do
not have any cost, but when $i\not\in C$ we like to allow the online
algorithm to ``stay'' in state $C$, which would involve fetching
page $i$ (unit cost) while evicting some page $j\in C$ and then
fetching back page $j\in C$ (another unit cost). This explains why
we charge two when $i\not \in C$.

The paging problem has been well studied as one of
the prototypical online problems. The best possible
competitive ratio for deterministic algorithms is $k$, and is
achieved by various algorithms~\cite{sleator1985amortized}
including Least Recently Used (LRU) and First In First Out (FIFO).
For randomized algorithms, the randomized marking algorithm
enjoys a competitive ratio of
$2H_k=2\sum_{i=1}^k\frac{1}{k}\approx 2\log(k)$, and is optimal up
to a multiplicative factor of $2$~\cite{fiat1991competitive}. By
Corollary \ref{cor:MTS_competitie} we have

\begin{corollary}\label{cor:paging_no_comp}
There is a paging algorithm such that
\begin{itemize}
\item Its competitive ratio is $2H_k$
\item Its regret on any interval $I$ is
\[
O\left(k\sqrt{|I|\log\left(T\right)} + \sqrt{k|I|\log\left(\binom{n}{k}T\right)} \right)=
O\left(k\sqrt{|I|\log\left(nT\right)}\right)
\]
\end{itemize}
\end{corollary}
Let us rephrase Corollary~\ref{cor:paging_no_comp}  in terms of the
paging problem. It guarantees an online algorithm for
which the number of cache misses is at most a $2H_k$ larger than what is achieved by the optimal offline schedule. Likewise, for long enough
segments, i.e., longer than $\Omega(k^2\log(nT))$, it guarantees that
the number of cache misses is not much larger compared to the best
single ``fixed'' cache $C^*$. Recall that for a fixed cache $C^*$,
whenever we have a request for a page $i\not\in C^*$ we first fetch
$i$ evicting $j\in C^*$ and then evict $i$ and fetch back $j$, for a
total cost of two.


We note that there are many natural cases where a fixed cache is
either optimal or near optimal. For example, if the page requests
are distributed i.i.d. then there is a cache configuration $C^*$
which minimizes the probability of a cache miss. Namely, the
configuration that has in the cache the $k$ pages with highest
probability. This implies that this fixed cache strategy is optimal,
up to a multiplicative factor of $2$ in this distributional setting.
Note that we do not need a single cache configuration for all $T$
time steps, but rather only for the interval $I$.


As discussed before, the main drawback of a naive application of
Corollary \ref{cor:MTS_competitie}, i.e.,
Corollary~\ref{cor:paging_no_comp}, is the running time of the
algorithm which scales with the number of states $N=O(n^k)$. Our
main additional contribution for the paging problem is to make the
online algorithm efficient. Namely, we show the following theorem.

\begin{theorem}\label{thm:paing_non_adaptive_intro}
There is an online paging algorithm with per-round runtime of
$\poly(n,\log(T))$ that enjoys a regret of $O(\sqrt{kT\log(n)})$.
\end{theorem}
Based on the algorithm of Theorem \ref{thm:paing_non_adaptive_intro}
and a variant of Corollary \ref{cor:MTS_competitie} (Theorem
\ref{thm:main_combine}) we obtain an efficient algorithm with
guarantees as in Corollary \ref{cor:paging_no_comp}.\footnote{Note
that that compared to theorem \ref{thm:paing_non_adaptive_intro} there is are slight differences in the regret term. The
additional factor of $O(\sqrt{k})$ arises from the need to maintain
the competitive ratio of $2H_k$. The additional logarithmic factor
in $T$ is needed in order to support the adaptivity over all
interval.}

\begin{corollary}\label{cor:pagin_competitive}
There is a paging algorithm with per-round runtime of
$\poly(n,\log(T))$ such that
\begin{itemize}
\item Its competitive ratio is $2H_k$
\item Its regret on any interval $I$ is at most
$O\left(k\sqrt{|I|\log\left(nT\right)}\right)$
\end{itemize}
\end{corollary}

{\bf $k$-server.}
In the online $k$-server problem there is a set $X$ of $n$ locations
and a set of $k$ servers, each located at some location $x\in X$.
There is a metric define over $X$, namely, $d(x_1,x_2)$ is the
distance between the locations $x_1,x_2\in X$. We assume that for
all $x_1,x_2\in X$, $d(x_1,x_2)\le 1$. At each time $t$ we have a
request for a location $x_t\in X$. If there is a server located at
$x_t$ we have a zero cost, and otherwise we have to move one of the
servers to location $x_t$. The cost of the online algorithm is the
sum of the distances the servers have traversed. We note that paging
is a special case of the $k$-server problem where $X$ is the uniform
metric space (i.e., $d(x,y)=1[x\ne y]$).

Again, it is fairly straightforward to model the $k$-server problem
as an MTS. The number of states of the MTS will be all
possible locations of the $k$ servers, i.e., $X^k$. This implies
that the number of states of the MTS is $N=n^k$. Again, note that
the number of states is exponential in $k$.

We now need to define a metric between the states and possible loss
functions. First we extend the distance function $d(\cdot,\cdot)$ to
configuration in $X^k$. Given two configurations $C_1,C_2\in X^k$
let $d(C_1,C_2)$ is the minimum weight matching between the $k$
locations in $C_1$ and the $k$ locations in $C_2$, where the weight between two
locations $x_1,x_2\in X$ is their distance $d(x_1,x_2)$.
By definition we can move from configuration $C_1$ to $C_2$ having
cost $d(C_1,C_2)$ by utilizing the minimum weight matching. Now we
need to define the possible loss functions. We have a possible loss
function $\ell_x$ for each location $x\in X$.
For a configuration $C\in X^k$, if $x\in C$ then $\ell_x(C)=0$ and
if $x\not\in C$ we have $\ell_x(C)=2 d(C,x)$, where
$d(C,x)=\min_{y\in C} d(y,x)$. Clearly when $x\in C$ we do not have
any cost, but when $x\not\in C$ we like to allow the online
algorithm to ``stay'' in configuration $C$, which would involve
moving a server to location $x$ (distance $d(C,x)$) and back (another
distance $d(C,x)$). This explains why we charge $2d(C,x)$ when
$x\not \in C$.



The $k$ server problem has been extensively studied in the online
algorithms literature. The best known competitive
ratio~\cite{koutsoupias1995k} for a deterministic algorithm is
$2k-1$. As for lower bounds, the best known is the paging lower
bound of $k$, and it is conjectured to be tight. As for randomized
algorithms the best known lower bound is again the paging lower
bound of $H_k$. As for upper bounds, a recent result by
\cite{Bubeck2017} improved
\cite{bansal2011polylogarithmic} and showed an
$O(\log^2(k))$-competitive algorithm. By Corollary
\ref{cor:MTS_competitie} we get that

\begin{corollary}\label{cor:k_server}
For any metric space $X$ there is a $k$-server algorithm such that
\begin{itemize}
\item Its competitive ratio is $\min\left(2k-1,O(\log^2(k))
\right)$
\item Its regret on any interval $I$ is
$O\left(k\sqrt{|I|\log\left(T\right)} + \sqrt{k|I|\log\left(n^kT\right)} \right)=
O\left(k\sqrt{|I|\log\left(nT\right)}\right)$
\end{itemize}
\end{corollary}

As in the case of paging, as the number of configurations is $n^k$,
and a naive application of Corollary  \ref{cor:MTS_competitie} will
result with a rather inefficient algorithm. Unlike paging problem,
we are unable to derive an efficient online algorithm with similar
guarantees. We are able to show that exhibiting such an online
algorithm would have interesting implications. In section \ref{sec:k_server} we show that
an online algorithm which achieves such guarantees and that runs in time
polynomial in $n$ will result with a $2$-approximation
algorithm for the $k$-median problem.
\begin{theorem}\label{thm:k-server-k-median}
If there is an online algorithm for the $k$-server problem
that has regret $poly(n)T^{1-\mu}$, then for every $\epsilon>0$ there is an
polynomial time algorithm that achieves a $2+\epsilon$ approximation
for the $k$-median problem.
\end{theorem}
Algorithm as in Theorem~\ref{thm:k-server-k-median} will
 improve on the best known
$2.675$-approximation ratio~\cite{byrka2015improved} of this well
studied problem.
It is worth noting that the best known hardness of approximation
results~\cite{jain2002new} for $k$-medians only rule out
approximation ratio of $1+\frac{2}{e}-\epsilon\approx 1.7358$.

\section{Overview over the Proofs}\label{sec:proof-overview}

In this section we sketch our algorithms and the
correctness proof, namely proving
Theorem~\ref{thm:main_non_uniform}.

{\bf Reducing to linear optimization.}
Our first step is to consider an equivalent continuous version of
the $N$-expert problem, which previously appeared in the literature
in the context of MTS (see for instance \cite{blum2000line}). It
turns out that the $N$-experts $D$-switching cost problem is
equivalent to the following linear optimization problem. At each
step $t=1,\ldots,T$ the player chooses $x_t\in \Delta^N$. Then, the adversary chooses a loss vector $l_t\in [0,1]^N$ and the player suffers a loss of $\inner{l_t,x_t} + D\|x_t-x_{t-1}\|_{TV} = \sum_{i=1}^Nl_t(i)x_t(i) +
    \frac{D|x_t(i)-x_{t-1}(i)|}{2}$, where $\|\cdot\|_{TV}$ is the total-variation distance.

The intuition is that at time $t-1$ the player has an action
distributed according to $x_{t-1}$. At time $t$ the player needs to
sample an action from $x_t$. This can be done in a way that the
probability of making a switch is exctly $\|x_t-x_{t-1}\|_{TV}$.
%

{\bf An overview.}
We first discuss why previous approaches fail to achieve our goal.
Previous results for the switching costs setting were proved by
using an algorithm whose number of switches, on every input
sequence, is upper bounded by the desired regret bound. In this
case, the switching costs can be simply absorbed in the regret term.
We next show that this strategy cannot work in the strongly adaptive
setting.
To this end, let us consider the following scenario.
There are two experts and $D=1$. In the first $T^{\frac{1}{10}}$
steps the first expert has $0$ loss, while the second has a loss of
$1$ at each step. Then, this is flipped every $T^{\frac{1}{10}}$
rounds. Namely, the loss sequence is
\[
\underbrace{(1,0), (1,0), \ldots, (1,0)}_{T^{\frac{1}{10}}\text{ times}},\;\;
\underbrace{(0,1), (0,1), \ldots, (0,1)}_{T^{\frac{1}{10}}\text{ times}},
\ldots,
\underbrace{(1,0), (1,0), \ldots, (1,0)}_{T^{\frac{1}{10}}\text{ times}},
\]
Suppose now that $\ca$ is a strongly adaptive algorithm operating on
this sequence. Concretely, let us assume that $\ca$ is guaranteed to
have a regret of $\sqrt{|I|\log(T)}$ on every time interval
$I\subset [T]$. Let us first consider the operation of the algorithm
in blocks of the form $I_i = [iT^{\frac{1}{10}} +
1,(i+1)T^{\frac{1}{10}}]$. In each block $I_{2j}$, $\ca$ will have
to give the first expert a weight $\ge \frac{3}{4}$ at least once.
Indeed, otherwise, its regret would be $\ge
\frac{1}{4}T^{\frac{1}{10}}\gg \sqrt{T^{\frac{1}{10}}\log(T)}$.
Similarly, in each block $I_{2j+1}$, $\ca$ will have to give the
second expert a weight $\ge \frac{3}{4}$ at least once. Let us now
consider the number of switches during the the entire run. By the
arguments above, the number of switches $\ca$ makes will be at least
$\frac{1}{2}T^{\frac{9}{10}}\gg \sqrt{T\log(T)}$.

In this paper we take a different approach, based on the algorithm
of \cite{kapralov2011prediction} for the two experts problem
without switching costs, which achieves essentially zero regret
w.r.t. to the first expert, while still maintaining optimal
asymptotic regret w.r.t. the second, i.e., $O(\sqrt{T})$.
Concretely, given a parameter $0<Z<\frac{1}{e}$, the regret w.r.t.
the first expert is $ZT$, while the regret w.r.t. the second is
$\sqrt{64T\log\left(\frac{1}{Z}\right)} + ZT +4$. For $Z=\sqrt{(\log
T)/T}$ we get the usual regret $O(\sqrt{T\log T})$, but we will aim
for $Z=1/T$ which will have a regret of $1$ to the first expert and
a regret of $O(\sqrt{T\log T})$ to the second expert.

We extend the analysis of \cite{kapralov2011prediction} and show
that the algorithm has similar regret bound even in the presence of
switching costs. Furthermore, we show that a variant of this
algorithm (obtained  by adding a certain projection) enjoys such a
regret bound on any time interval. Concretely, we prove that
\begin{theorem}\label{thm:sketch_two_experts}
There is an algorithm for the $2$-experts $D$-switching cost
problem, that given parameters $Z\le\frac{1}{e}$ and $\tau\ge 1$
has the following regret bounds
\begin{itemize}
\item
For any time interval $I\subseteq [T]$, the regret of the algorithm
w.r.t.\ expert $0$ is at most\\
\mbox{$\min\left\{\sqrt{D}TZ, \sqrt{16 D \tau\log\left(\frac{1}{Z}\right)}+2\sqrt{D}+\sqrt{D}|I|Z\right\}$}
\item
For every time interval $I\subseteq [T]$ of length $\le \tau$, the
regret of the algorithm
w.r.t. expert $1$ is at most
$\sqrt{64D\tau\log\left(\frac{1}{Z}\right)}+4\sqrt{D}+ \sqrt{D}\tau
Z$
%
\end{itemize}
\end{theorem}
To better understand Theorem~\ref{thm:sketch_two_experts}, consider
setting $Z=1/(\sqrt{D}T)$ and $\tau=T$:
\begin{corollary}\label{cor:sketch_two_experts}
There is an algorithm for the $2$-experts $D$-switching cost
problem, that has the following regret bounds
\begin{itemize}
\item
For any time interval $I\subseteq [T]$, the regret of the algorithm
w.r.t.\ expert $0$ is at most $1$.
\item
The regret of the algorithm
w.r.t. expert $1$ is at most $O(\sqrt{DT \log(TD)})$
%
\end{itemize}
\end{corollary}

Theorem~\ref{thm:sketch_two_experts} is the main building block for
proving Theorem \ref{thm:main_non_uniform}. We first use it to
combine two algorithms. Namely, given two algorithms $\ca_0$ and
$\ca_1$ for the $N$-experts $D$-switching cost problem, we use
Theorem \ref{thm:sketch_two_experts} to combine them into a single
algorithm the preserves regret bounds of $\ca_0$ and $\ca_1$, plus
additional quantities, as in Theorem \ref{thm:sketch_two_experts}.
Then, we use this basic combining procedure to combine many
algorithm, deriving Theorem \ref{thm:main_non_uniform}.

{\bf A sketch of Theorem \ref{thm:sketch_two_experts}'s proof.}
In this section we highlight the main ideas in the proof of
Theorem~\ref{thm:sketch_two_experts}, and make a few simplifying
assumption to help the presentation, including that $D=1$.
At time step $t\ge 1$, the player chooses $g(x_t)\in [0,1]$ and suffers
a loss of $\ell_t^{on}=\ell_t(0)(1-g(x_t))+\ell_t(1)g(x_t) +
D|g(x_t)-g(x_{t-1})|$.
The algorithm has two parameters: $\tau \ge  1$ and $Z>0$.
Define $\tilde{g}$ to be the solution of the differential equation
\begin{equation}\label{eq:diff_eq_sketch}
8\tilde{g}'(x)= \frac{1}{\tau} x \tilde{g}(x) +Z,\;\;\;\;\tilde{g}(0)=0~.
\end{equation}
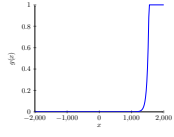
\begin{figure}
\caption{$g(x)$ for $Z=10^{-8}$ and $\tau = 10000$}\label{fig:function_plot}
\begin{center}
\begin{tikzpicture}[scale=0.25]
\begin{axis}[
    axis lines = left,
    xlabel = $x$,
    ylabel = {$g(x)$},
]
\addplot [
    domain=0:2000,
    samples=100,
    color=blue,
    thick
]
{max (0,min (1,sqrt(10000/8)*0.00000001*exp ((x)^2/160000)))};
\addplot [
    domain=-2000:0,
    samples=100,
    color=blue,
    thick
]
{0};
 \end{axis}
\end{tikzpicture}
\end{center}
\end{figure}
Define $U=U_{\tau,Z}:=\tilde{g}^{-1}(1)$ ($\tilde{g}$ is strictly
increasing and unbounded, so $U$ is well defined, and we later show
that $U\le \sqrt{16\tau\log\left(\frac{1}{Z}\right)}$). Denote the
projection to $[a,b]$ by
$\Pi_{[a,b]}(z)=
\begin{cases}
a & z \le a \\
z &  a \le z \le b \\
b & z \ge b
\end{cases}$.
Finally, we define $g(x)=\Pi_{[0,1]}\left[\tilde{g}(x)\right]$,
which ensures that $0\leq g(x)\leq 1$ and that for $x\leq 0$ we have
 $g(x)=0$ and for $x\geq U$ we have $g(x)=1$. (see Figure
\ref{fig:function_plot} for the plot of $g(x)$).
The following algorithm is a simplification of Algorithm
\ref{alg:two_with_proj} that achieves the regret bounds of Theorem
\ref{thm:sketch_two_experts}.
\begin{algorithm}[H]
    \caption{Two experts (with parameters $\tau$ and $Z$)} \label{alg:sketch_two_with_proj}
    \begin{algorithmic}[1]
        \STATE Set $x_t=0$
        \FOR {$t=1,2, \ldots$}
        \STATE Predict $g(x_t)$
        \STATE Let $b_t=l_t(0)-l_t(1)$ and update $x_{t+1}=
         \left(1-\frac{1}{\tau}\right) x_t + b_t$
        \ENDFOR
    \end{algorithmic}
\end{algorithm}
We next elaborate on the proof. To highlight the main ideas, we will
sketch the proof of just for two special cases: (1) showing that the
regret to the second expert for the time interval $[1,\tau]$ is at
most $\sqrt{64\tau\log\left(\frac{1}{Z}\right)}+2+\tau Z$, and (2)
For the time interval $[T]$, the regret of w.r.t. the first expert
is at most $TZ$.
%
We need the following helpful notation. Let $G(x)=\int_0^xg(s)ds$,
denote $\Phi_t=G(x_t)$ and let $I:\reals\to \{0,1\}$ be the
indicator function of the segment $[-2,U+2]$. We consider the change
in the value of $\Phi_t$,
\begin{eqnarray*}
    \Phi_{t+1}-\Phi_{t}&=&\int_{x_t}^{x_{t+1}} g(s)ds = \int_{x_t}^{x_t-\frac{1}{\tau} x_t+b_t}g(s)ds
    \\
    &\le & g(x_t)\left(-\frac{1}{\tau} x_t+ b_t\right)+\frac{1}{2}\left(-\frac{1}{\tau} x_t+ b_t\right)^2\max_{s\in [x_t,x_t-\tau^{-1}x_t+ b_t]}|g'(s)|
    \\
    &\le & g(x_t)\left(-\frac{1}{\tau} x_t+ b_t\right)+\frac{1}{2}4\max_{s\in [x_t,x_t-\tau^{-1}x_t+ b_t]}|g'(s)|
    \\
    &= & g(x_t)\left(-\frac{1}{\tau} x_t+ b_t\right)-2\max_{s\in [x_t,x_t-\tau^{-1}x_t+ b_t]}|g'(s)| + 4\max_{s\in [x_t,x_t-\tau^{-1}x_t+ b_t]}|g'(s)|
    \\
    &\le & g(x_t)\left(-\frac{1}{\tau} x_t+ b_t\right)- 2 \frac{|g(x_t)-g(x_{t+1})|}{\left|-\frac{1}{\tau} x_t+ b_t\right|}+4\max_{s\in [x_t,x_t-\tau^{-1}x_t+
    b_t]}|g'(s)|
    \\
    &\le & g(x_t)\left(-\frac{1}{\tau} x_t+ b_t\right)- |g(x_t)-g(x_{t+1})|+4\max_{s\in [x_t,x_t-\tau^{-1}x_t+ b_t]}|g'(s)|
    \\
    &\le & g(x_t)\left(-\frac{1}{\tau} x_t+ b_t\right)- |g(x_t)-g(x_{t+1})|+\frac{1}{\tau}x_tg(x_t)I(x_t)+Z
\end{eqnarray*}
Here, the first inequality follows from the fact that for every
piece-wise differential function $f:[a,b]\to\reals$ we have
$\int_a^bf(x)dx\le f(a)(b-a) + \frac{1}{2}(b-a)^2\max_{\xi\in
[a,b]}|f'(\xi)|$. The second and forth inequalities follows from the
fact that by a simple induction, $|x_t|\le \tau$ and hence
$\left|-\frac{x_t}{\tau}+b_t\right|\le 2$. The third inequality
follows from the mean value Theorem. As for the last inequality,
since $g$ is a solution of the ODE from equation
(\ref{eq:diff_eq_sketch}), and is constant outside $[0,U]$ we have
$4g'(x_t)\le
\frac{1}{2}\left[\frac{1}{\tau}x_tg(x_t)I(x_t)+Z\right]$. We later
show that $g$ is smooth enough so that the inequality remains valid,
up to a factor of $4$, on the entire interval $[x_t, x_t -
\tau^{-1}x_t + b_t]$. Namely, $4g'(x_t)\le
\frac{1}{\tau}x_tg(x_t)I(x_t)+Z$. Summing from $t=1$ to $t=T$,
rearranging, and using the fact that $\Phi_1=0$ we get,
\begin{equation}\label{eq:ym1}
-\sum_{t=1}^T g(x_t) b_t+ \sum_{t=1}^T |g(x_t)-g(x_{t+1})| \le
 - \Phi_{T+1}
+\sum_{t=1}^T\frac{1}{\tau}x_tg(x_t)(I(x_t)-1)+TZ
\end{equation}
Now, denote the loss of the algorithm up to time $t$ by $L_t^{on}$,
and by $L_t^i$ the loss of the expert $i$. We have
\begin{eqnarray}\label{eq:0_sketch}
L_t^{on} &=& \sum_{t'=1}^t g(x_{t'})l_{t'}(1)
+\sum_{t'=1}^t(1-g(x_{t'}))l_{t'}(0)+ \sum_{t'=1}^t
|g(x_{t'})-g(x_{{t'}+1})| \nonumber
\\
&=& L_t^0 + \left(-\sum_{t'=1}^tg(x_{t'})b_{t'} + \sum_{t'=1}^t
|g(x_{t'})-g(x_{t'+1})|\right) \nonumber
\\
&\le& L_{t}^0 -
\Phi_{t+1}+\sum_{t'=1}^t\frac{1}{\tau}x_{t'}g(x_{t'})(I(x_{t'})-1)+tZ,
\end{eqnarray}
where the inequality follows from~(\ref{eq:ym1}).
When considering the entire interval $[T]$, we can derive the
following regret. Since $\Phi_{T+1}\ge 0$ and $x_{t'}
g(x_t)(I(x_{t'})-1)\le 0$ we have $L^{on}_T\le L_T^0 + TZ$, which
proves that the regret of the first expert over $[T]$.
We now prove the regret bound for the second expert, in our special
case, i.e., to bound $L_\tau^{on} - L_\tau^1$.
Recall that $g(z) = 1$ any $z \ge U$, $0\leq g(z)\leq 1$ for $z \in
[0,U]$, and $g(z)=0$ for any $z\le 0$. Therefore, we have
$\Phi_{\tau+1}=\int_{0}^{x_{\tau+1}}g(s)ds\ge x_{\tau+1}-U$. Also,
$x_tg(x_t)(1-I(x_t)) \ge x_t - U -2 $, since for $z\geq U$ we have
$g(z)=1$. Therefore, if we denote $b_0:=0$, we have
\small
\begin{eqnarray*}
    \Phi_{\tau+1}+\sum_{t=1}^\tau\frac{1}{\tau}x_tg(x_t)(1-I(x_t)) &\ge & (x_{\tau+1}-U) +\sum_{t=1}^\tau\frac{1}{\tau}(x_t-U -2)
    \\
    &=& \sum_{j=0}^{\tau}\left(1-\frac{1}{\tau}\right)^{\tau-j}b_j+\sum_{t=1}^\tau\frac{1}{\tau}\sum_{j=0}^{t-1}\left(1-\frac{1}{\tau}\right)^{t-1-j}b_j- 2(U+1)
    \\
    &=& \sum_{j=0}^{\tau}\left(1-\frac{1}{\tau}\right)^{\tau-j}b_j+\frac{1}{\tau}\sum_{j=0}^{\tau-1}\sum_{t=j+1}^\tau\left(1-\frac{1}{\tau}\right)^{t-1-j}b_j-
    2(U+1)
    \\
    &=& \sum_{j=0}^{\tau}\left(1-\frac{1}{\tau}\right)^{\tau-j}b_j+\frac{1}{\tau}\sum_{j=0}^{\tau}\frac{1-\left(1-\frac{1}{\tau}\right)^{\tau-j}}{\frac{1}{\tau}}b_j- 2(U+1)
    \\
    &=& \sum_{j=0}^{\tau}b_j- 2U =  L_\tau^0-L_\tau^1 - 2(U+1)
\end{eqnarray*}
\normalsize
The above shows that
$L_\tau^1 + 2(U+1) \geq L_\tau^0 -\Phi_{\tau+1}
+\sum_{t=1}^\tau\frac{1}{\tau}x_tg(x_t)(I(x_t)-1)$.
By equation (\ref{eq:0_sketch}), $ L_\tau^1+2(U+1) +\tau Z\ge
L^{on}_{\tau}$. The proof is concluded by showing that $U\le
\sqrt{16\tau\log\left(\frac{1}{Z}\right)}$.

\subsection{Metrical Tasks  Systems}

We will derive our results for MTSs from the following variant of
Theorem \ref{thm:main_non_uniform}
\begin{theorem}\label{thm:main_combine}
There is a procedure that given as input experts algorithms
$\ca_\base,\ca_0,\ldots,\ca_{\log_2(T)}$, combines them into a
single algorithm $\ca$ such that:
\begin{enumerate}
\item
On  any interval $I$ of length $\frac{T}{2^{u+1}} \le |I| \leq\frac{T}{2^{u}}$,
the regret of $\ca$ w.r.t. $\ca_u$ is
$O\left(D\sqrt{|I|\log\left(T\right)}\right)$
\item
The regret of $\ca$ w.r.t. $\ca_\base$ is $D$
\end{enumerate}
Furthermore, if the original algorithms are efficient, then so is
$\ca$. More precisely, at each round the procedure is given as input
the loss of each algorithm in that round, and an indication which
algorithms made switches. Then, the procedure specifies one of the
experts algorithms, and $\ca$ chooses its action. The computational
overhead of the procedure is
$O\left(\log\left(T\right)\right)$.\footnote{Assuming that
arithmetic operations, exponentiation, computing the error function,
and sampling uniformly from $[0,1]$ cost $O(1)$.}
\end{theorem}
The combination of the various algorithms is done as follows. We start with $\ca_\base$ and combine it, using the two experts algorithm, with $\ca_0$. This yields an algorithm $\cb_0$ with essentially no regret w.r.t.\ $\ca_\base$ and small regret w.r.t.\ $\ca_0$. Then, we continue doing so, and at step $i$, we combine $\cb_i$ with $\ca_{i+1}$ to obtain $\cb_{i+1}$.

To prove Corollary \ref{cor:MTS_competitie} (which also implies
Corollaries \ref{cor:paging_no_comp} and \ref{cor:k_server}), we
take $\ca_\base$ be an algorithm with competitive ratio of $\alpha$,
and $\ca_0,\ldots,\ca_{\log_2 T}$ to be the algorithm from Corollary
\ref{cor:MTS_no_competitie}. Likewise, we set the switching costs to
be the diameter of the underlying metric space.
We run algorithm $\ca_u$ with parameter $\tau_u=2^{-u}T$ and set
$Z_u=1/T$.

{\bf Paging.}
To prove Corollary \ref{cor:pagin_competitive} we take $\ca_\base$
to be some paging algorithm with competitive ratio $2H_k$ (say, the
marking algorithm \cite{fiat1991competitive}). For $0\le u\le \log_2
T$, $\ca_u$ is obtained by sequentially applying an algorithm with
time horizon $2^{-u}T$, that enjoys the regret bound from Theorem
\ref{thm:paing_non_adaptive_intro}, and has per round running time
$\poly(n,\log(T))$. Likewise, we set the switching costs to be $k$.
The corollary follows easily from Theorem \ref{thm:main_combine}
when $I$ has the form $[(j-1)2^{-u}T + 1,j2^{-u}T]$. Otherwise, by a
standard chaining argument (e.g., \cite{daniely2015strongly}), we
can decompose $I$ into a disjoint union on segments $I_0,\ldots,I_l$
of the form $[(j-1)2^{-u}T + 1,j2^{-u}T]$ such that $|I_{i}|\le
|I|2^{-\frac{i-1}{2}}$. Summing the regret bound on the different
segments yields a geometric sequence, establishing the corollary.

In section \ref{sec:paging} we show that the multiplicative weights
algorithm, applied to the paging problem, enjoys the regret bounds
from Theorem \ref{thm:paing_non_adaptive_intro}. Furthermore, we
show that despite the exponential number of experts, it can be
implemented efficiently. The basic idea is the following. The MW
algorithm maintains a positive weight for each expert, and at each
step predicts the probability distribution obtained by dividing each
weight by the sum of the weights. We note that in paging, due to the
special structure of the loss vectors, after the page requests
$i_1,\ldots,i_t$, the weight of the expert corresponding to the
cache $A\in \binom{[N]}{k}$ is $\prod_{t'=1}^t e^{\eta 1[i_{t'}\in
A]}$ where $\eta=\sqrt{\frac{\log\left(\binom{n}{k}\right)}{kT}}$.
This special structure enables the use of dynamic programming in
order to implement the MW algorithm efficiently.


\section{Future Directions}\label{sec:future}

We presented a regret minimization methodology to classic online
computation problems, which can interpolate between a static
benchmark to a dynamic one. An important building block was
developing an expert algorithm that is strongly adaptive even in the
presence of switching costs.  As elaborated below, our work leaves
many open directions.

Maybe the most interesting and fruitful direction is to find other
online problems that can fall into our framework and for which we
can develop computationally efficient online algorithms. We believe
that our results can be extended the many other online problems such
as competitive data structures (see~\cite{blum2002static} for
results in this direction), buffering, scheduling, load balancing,
etc. Likewise, our expert results can be extended to other settings
such as partial information models\footnote{We remark that in the
classic bandit setting strong adaptivity is impossible, even without
switching costs~\cite{daniely2015strongly}.}, strategic
environments, non-oblivious adversaries, etc.
In our approach there is a design issue of selecting the benchmark
to consider. We have taken probably the most obvious benchmark, an
online algorithm that does not change its state, however, one can
consider other benchmarks which are more problem specific. In
addition to this grand challenge, there are also open problems
regarding our specific analysis.

While most of our regret bounds are tight, some regret bounds have
certain gaps, which would be interesting to overcome.
For MTS, the regret bound for the algorithm that ensures optimal
competitive ratio (Corollary \ref{cor:MTS_competitie}) is worst by a
factor of $\sqrt{D}$ compared to the algorithm without this
guarantee (Corollary \ref{cor:MTS_no_competitie}). We wonder if this
gap can be reconciled. For paging, we conjecture that there are
efficient algorithms with regret bound $O(\sqrt{kT\log(N)})$, which
is better than our bound of $O(k\sqrt{T\log(N)})$. For $k$-server,
we conjecture that it is NP-hard to efficiently achieve a regret
bound of $O(\sqrt{\poly(k,D)T})$. On the other hand, we conjecture
that there is a constant $c>1$ such that there are efficient
algorithms whose cost is at most $c$ times the costs of the best
fixed-locations strategy, plus $O(\sqrt{\poly(k,D)T})$. We note that
since there are $(2k-1)$-competitive deterministic and $
O\left(\log^2(k)\right)$-competitive randomized algorithms for
$k$-sever, our conjecture is true if $c$ is not a constant, but
rather $(2k-1)$ or $O\left(\log^2(k)\right)$.
In addition to improving our bounds, obtaining simpler algorithms with similar guarantees would be of great interest.
Specifically, for paging, $k$-server and MTS, the obtained algorithms are somewhat cumbersome and obtained by combining many algorithms
(especially when we insist on ensuring optimal competitive ratio).

Finally, regarding our expert algorithm and the tightness of the
regret bounds in the case of switching cost $1$, Theorem
\ref{thm:main_non_uniform} is tight up to a constant factor.
Yet, for general $D\ge 1$, there is a larger gap. While our
regret bound is $O\left(\sqrt{D|I|\log\left(NT\right)}\right)$, the
best known lower bound (that is obtained by combining Theorem
\ref{thm:main_lower} with \cite{geulen2010regret}) is
$\Omega\left(\sqrt{D|I|\log\left(N\right)}+\sqrt{|I|\log\left(NT\right)}\right)$.
Another interesting direction is to find a simpler algorithm and
analysis for Theorem \ref{thm:main_non_uniform}.



\section{Online Linear Optimization with Switching Cost}
\subsection{Notations and conventions.}
For simplicity, we assume throughout that $T=2^K$ and
that the switching cost $D\ge 1$.\footnote{This is justified
since our bounds in the case that $D<1$ are the same as the case
$D=1$, and the results for $D<1$ can be obtained by a simple
reduction to the case $D=1$.}

The letter $\tau$ will be used to denote the length of time
intervals $I\subset [T]$, and therefore by convention $\tau$ is an
integer that is greater $0$.

We denote by $\Delta(N)$ the simplex over a set $N$ elements.
Unless otherwise stated, the norm over points in the simplex is the total variation
norm, i.e.,
$\|\cdot\|=\frac{\|\cdot\|_1}{2}$. For $x\in\reals$ we denote $x_+ =
\max(0,x)$ and for $x\in \reals^d$ we let $x_+ =
((x_1)_+,\ldots,(x_d)_+)$.

For a piece-wise differentiable function $f:(a,b)\to\reals$ and
$x\in (a,b)$ we use the convention that $|f'(x)|$ denotes the
maximum over the right and left derivatives of $f$ at $x$. For
integrable function $f:[a,b]\to\reals$ we denote
$\int_{b}^af(x)dx:=-\int_a^b f(x)dx$. For a segment $I=[a,b]$ we let
$\Pi_{I}(x)$ be the projection on $I$, namely,
$$\Pi_{I}(x)=\begin{cases}b & x> b\\ x & a\le x\le b \\ a &
x<a\end{cases}$$

\subsection{From $N$-experts $D$-switching cost  to online linear optimization with switching cost}
Our first step is to reduce the $N$-experts $D$-switching cost
problem
(abbreviated $\EXP$) to the problem of online linear optimization
over $\Delta(N)$ with total variation switching cost and parameter
$D\ge 1$ (abbreviated $\OLO$). The latter problem is defined as
follows. The game is played for $T$ time steps, such that at each
time step $t=1,2,\ldots,T$
\begin{itemize}
    \item Nature chooses a loss $l_t\in [0,1]^N$
    \item The learner choose an action $x_t\in \Delta(N)$
    \item The player suffers a loss of $\inner{l_t,x_t} + D\|x_t-x_{t-1}\|$ (or just $\inner{l_t,x_t}$ if  $t=1$)
\end{itemize}
We assume that nature is oblivious (i.e., the loss sequence was
chosen before the game started) and that the learner is
deterministic and its action at step $t$ depends only on the past
losses $l_1,\ldots,l_{t-1}$. The notion of regret is define in Section~\ref{sec:results}.

As we will show in this section, $\OLO$ is essentially equivalent to
$\EXP$. Namely, we will show that any algorithm for $\OLO$ can be
transformed to an algorithm for $\EXP$ such that for every loss
sequence, the loss (expected loss in the case of $\EXP$) remains the
same. Likewise, any algorithm for $\EXP$ can be transformed to an
algorithm for $\OLO$ such that for every loss sequence, the loss
does not grow.

We start by reducing $\EXP$ to $\OLO$. Let $\ca$ be an algorithm for
$\OLO$. We will explain how to transform it to a learning algorithm
$\ca'$ for $\EXP$. Let $\cl=\{l_1,\ldots,l_T\}\subset [0,1]^N$ be a
loss sequence, and let $x_1,\ldots,x_T\in \Delta(N)$ be the
distributions over actions of $\ca$ on that sequence. In Lemma
\ref{lem:switch_to_l1}
we show that for every pair of consecutive distributions over
actions $x_{t-1},x_t$, there is a transition probability matrix
$p_t(i|j)$, $i,j\in [N]$ such that if $X_{t-1}$ is distributed
according to $x_{t-1}$ and $X_t$ is generated from $X_{t-1}$ based
on $p_t$, then $X_t$ is distributed according to $x_{t}$ and
moreover, $\Pr\left( X_t\ne X_{t-1} \right) = \|x_t-x_{t-1}\|$.
Given this, we can construct an algorithm $\ca'$ such that the
expert chosen at round $1$ is a random expert $X_1$ distributed
according to $x_1$, and for each $t>1$, $X_t$ is generated from
$X_{t-1}$ according to the transition probability matrix $p_{t}$.
The above discussion shows that at each step $t$ the loss of $\ca'$
in $EXP(N,D)$ is
\[
\ymE l_t(X_t) + D\Pr\left(X_t\ne X_{t-1}\right) =
\inner{l_t,x_t}+D\|x_t-x_{t-1}\|\;,
\]
which is exactly the loss of $\ca$ at the same step in $OLO(N,D)$.

We first prove the following simple fact.

\begin{lemma}\label{fact:simplex_norm}
For every $z,z'\in\Delta(N)$, $\|z-z'\| = \sum_{j=1}^N (z'_j-z_j)_+
= \sum_{j=1}^N (z_j-z'_j)_+$, and $\frac{(z-z')_+}{\|z-z'\|}$ is a
distribution.
\end{lemma}
\begin{proof}
Since for any distribution the probabilities sum to $1$, we have
that $\sum_{i=1}^N z'_i-z_i = 0$. By splitting the last sum to
positive and negative summands we conclude that
\[
\sum_{i=1}^N (z'_i-z_i)_+ = \sum_{i=1}^N (z_i-z'_i)_+\;.
\]
Hence, considering the total variation norm, we have,
\[
\|z-z'\| = \frac{\sum_{i=1}^N |z'_i-z_i|}{2} = \frac{\sum_{i=1}^N ((z'_i-z_i)_+) + \sum_{i=1}^N(z_i-z'_i)_+}{2}= \sum_{j=1}^N (z'_j-z_j)_+
\]
The fact that $\frac{(z-z')_+}{\|z-z'\|}$ is a distribution follows
from the fact that $\sum_{j=1}^N \frac{(z_j-z'_j)_+ }{ \|z-z'\|}=1$.
\end{proof}
We first define the transition probability matrix.
\begin{definition}
For distributions $z,z'\in \Delta(N)$ we define
\[
    p_{z,z'}(i,j)=
    \min\{z_i,z'_j\}1[i=j]+\frac{((z_i-z'_i)_+)( (z'_j-z_j)_+)}{\|z-z'\|}
\]
\end{definition}
It remains to state and prove the lemma for the joint distribution.
\begin{lemma}\label{lem:switch_to_l1}
For any $z,z'\in \Delta(N)$, $p_{z,z'}$ is a distribution on $[N]\times [N]$. Furthermore, if $(X,X')$ is a random variable distributed according to $p_{z,z'}$  then
\begin{enumerate}
        \item $X$ is distributed according to $z$ and $X'$ is distributed according to $z'$, and
        \item $\Pr\left(X\ne X'\right)=\|z-z'\|$
\end{enumerate}
\end{lemma}
Note that if $z=z'$ then we have that $p_{z,z'}(i,j)=z_i 1[i=j]$, namely, the support is only the pairs $(i,i)$.

\begin{remark}\label{rem:sampling_min_cost}
Suppose that $X$ is distributed according to $z$ and we want to
generate $X'$ that distributed according to $z'$ and $\Pr\left(X\ne
X'\right)=\|z-z'\|$. We can sample as follows. Assume that $X=i$, then with probability
$\frac{(z_i-z'_i)_+}{z_i}$ sample $X'$ from the distribution
$\frac{(z'-z)_+}{\|z-z'\|}$ and w.p. $1-\frac{(z_i-z'_i)_+}{z_i}$
set $X':=X$.
\end{remark}

\proof (of Lemma \ref{lem:switch_to_l1})
Assume that $z\ne z'$. By Fact \ref{fact:simplex_norm}, we have that for every $i\in [N]$,
\begin{eqnarray}\label{eq:0.1}
\sum_{j=1}^Np_{z,z'}(i,j) &=&\min\{z_i,z'_i\}+\frac{1}{\|z-z'\|}((z_i-z'_i)_+)\sum_{j=1}^N (z'_j-z_j)_+ \nonumber
\\
&=&\min\{z_i,z'_i\}+\frac{1}{\|z-z'\|}((z_i-z'_i)_+)\|z-z'\|
\\
&=&\min\{z_i,z'_i\}+((z_i-z'_i)_+) =  z_i \nonumber
\end{eqnarray}
Similarly, for every $j\in [N]$,
\begin{equation}\label{eq:0.2}
\sum_{i=1}^Np_{z,z'}(i,j)=z_j,
\end{equation}
and,
\begin{equation}\label{eq:0.3}
\sum_{i=1}^N\sum_{j=1}^Np_{z,z'}(i,j) = \sum_{i=1}^Nz_i = 1\;.
\end{equation}
Equations (\ref{eq:0.1}), (\ref{eq:0.2}) and (\ref{eq:0.3}) show
that $p_{z,z'}$ is a distribution whose marginals are $z$ and $z'$.
It remains to prove that if $(X,X')$ is distributed according to
$p_{z,z'}$ then $\Pr\left(X\ne X'\right)=\|z-z'\|$. Indeed, again by
Fact \ref{fact:simplex_norm},
\[
\Pr\left(X\ne X'\right)= 1-\sum_{i=1}^N\min\{z_i,z'_i\}
= \sum_{i=1}^Nz_i-\min\{z_i,z'_i\}
= \sum_{i=1}^N(z_i-z'_i)_+=\|z-z'\|
\]
\proofbox

We now show how an $\ca'$ algorithm for $\EXP$ can be transformed to
an algorithm $\ca$ for $\OLO$ such that on any loss sequence, the
loss of $\ca$ in $\OLO$ is the same as the loss of $\ca'$ in
$EXP(N,D)$. To this end, let $\cl$ be a loss sequence, and let
$X_t\in [N]$ be the expert chosen by $\ca'$ when it runs on $\cl$.
Let $x_t\in \Delta(N)$ be the distribution of $X_t$. The algorithm
$\ca$ will simply play $x_t$ at round $t$. The loss of $\ca$ in
$\OLO$ at step $t$ is
\[
\inner{l_t,x_t} + D\|x_t-x_{t-1}\|,
\]
while the loss of $\ca'$ in $\EXP$ is
\[
\inner{l_t,x_t} + D\Pr\left(X_t\ne X_{t-1}\right)
\]
which is equal to the loss of $\ca$ by
Lemma~\ref{lem:switch_to_l1}.

\subsection{An algorithm for two experts}
Consider the $\OLO$ problem with$N=2$.
It will be convenient to use $x\in[0,1]$ to denote the probability
distribution $(x,1-x)$. Concretely, the game is defined as follows.
At each step $t\ge 1$,
\begin{itemize}
\item The adversary chooses $l_t=(l_t(0),l_t(1))\in [0,1]^2$.
\item The player chooses $z_t\in [0,1]$ and suffers a loss of
\[
l_t=l_t(0)(1-z_t)+l_t(1)z_t + D|z_t-z_{t-1}|
\]
(for $t=1$, assume that $z_0:=0$)
\item $l_t$ is revealed to the player.
\end{itemize}
Consider the following algorithm. Let $\tau, D\ge  1$ and $Z>0$. Let
$\ymerf(x)=\int_0^x\exp\left(-\frac{s^2}{2}\right)ds$.\footnote{Note
that this definition is slightly different form the standard
error-function that is defined as
$\frac{2}{\sqrt{\pi}}\int_0^x\exp\left(-s^2\right)ds$.} Define
$\tilde{g}=\tilde{g}_{\tau , Z}$ as follows
\begin{equation}\label{eq:def_eq}
\tilde{g}(x) =
\sqrt{\frac{\tau}{8}}Z\ymerf\left(\frac{x}{\sqrt{8\tau}}\right)\exp\left(\frac{x^2}{16\tau}\right)
\end{equation}
We note that $\tilde{g}$ is a solution of the differential equation
\begin{equation}\label{eq:diff_eq}
8\tilde{g}'(x)= \frac{1}{\tau} x \tilde{g}(x) +Z~.
\end{equation}
Define
$U=U_{\tau,Z}:=\tilde{g}^{-1}(1)$. Also, define $g=g_{\tau,Z}$ as
\[
g(x)=\Pi_{[0,1]}\left[\tilde{g}(x)\right]=\begin{cases}
0 & x \le 0
\\
\tilde{g}(x) &  0 \le x \le U_{\tau,Z}
\\
1 & x \ge U_{\tau,Z}
\end{cases}
\]
\begin{algorithm}[H]
\caption{Two experts (with parameters $\tau$, $D$ and $Z$)} \label{alg:two_with_proj}
\begin{algorithmic}[1]
\STATE Set $x_t=0$
\FOR {$t=1,2, \ldots$}
\IF{$D\log\left(\frac{1}{Z}\right) \le \frac{\tau}{64}$}
\STATE Predict $g(x_t)$
\ELSE
\STATE Predict $0$
\ENDIF
\STATE Let $b_t=\frac{l_t(0)-l_t(1)}{\sqrt{D}}$
\STATE Update $x_{t+1}=\Pi_{[-2,U+2]}\left[ \left(1-\frac{1}{\tau}\right) x_t +  b_t\right]$
\ENDFOR
\end{algorithmic}
\end{algorithm}

To parse the following theorem, we note that when we will apply it,
we will take $Z$ to be very small (say, $T^{-10}$). Likewise, $\tau$
can ba any number $0\le \tau \le T$, potentially much smaller than
$T$. In these settings, the theorem shows that the algorithm has
very small regret with respect to expert 0. Namely the regret is
$o(1)$
even on segments that are much larger than $\tau$. Remarkably, the
algorithm is able to achieve this while preserving almost optimal
$\tilde O\left(\sqrt{\tau}\right)$ regret w.r.t. expert 1, but only
on segments of length $\tau$.

\begin{theorem}\label{thm:two_experts}
Suppose $Z\le\frac{1}{e}$. Algorithm \ref{alg:two_with_proj}
guarantees,
\begin{itemize}
\item For every time interval $I$, the regret  w.r.t. expert $0$ is at most
$\min\{\sqrt{D}TZ,\sqrt{16 D
\tau\log\left(\frac{1}{Z}\right)}+2\sqrt{ D}+\sqrt{ D}|I|Z\}$
\item For every time interval $I$ of length $\le \tau$, the regret w.r.t. expert $1$ is at
most $\sqrt{64D\tau\log\left(\frac{1}{Z}\right)}+4\sqrt{D}+ \sqrt{
D}\tau Z$
\end{itemize}
\end{theorem}

\subsubsection{Proof of Theorem \ref{thm:two_experts}}

\paragraph{Properties of $g_{\tau,Z}$}
We first prove some properties of the function $\tilde{g}$.
\begin{lemma}\label{lem:basic_g}
The function $\tilde{g}(x)$ has the following properties:
\begin{enumerate}
\item $\tilde{g}(x)$ is strictly increasing odd function.
\item $\tilde{g}(x)$ is convex in $[0,\infty)$.
\item For $\tau\ge 8e$ and $Z\le\frac{1}{e}$ we have $U_{\tau,Z}\le
\sqrt{16\tau\log\left(\frac{1}{Z}\right)}$, where $U_{\tau,Z} :=
\tilde{g}^{-1}(1)$
\end{enumerate}
\end{lemma}
\proof Part 1 follows immediately form equation (\ref{eq:def_eq}).
For part 2, note that $\tilde{g}'(x)\geq 0$ and that
\begin{equation}\label{eq:scond_derivative}
8\tilde{g}''(x)= \frac{1}{\tau} \tilde{g}(x) + \frac{1}{\tau}x\tilde{g}'(x) ~.
\end{equation}
Hence, $\tilde{g}''$ is non-negative in $[0,\infty)$ and therefore
$\tilde{g}$ is convex. For Part 3, we have
\begin{eqnarray*}
\tilde{g}\left(\sqrt{16\tau\log\left(\frac{1}{Z}\right)}\right) &=&
\sqrt{\frac{\tau}{8}}Z\ymerf\left(\sqrt{2\log\left(\frac{1}{Z}\right)}\right)\frac{1}{Z}
\\
&\ge & \sqrt{\frac{\tau}{8}}\ymerf\left(1\right)
\\
&= & \sqrt{\frac{\tau}{8}}\int_0^1 e^{-\frac{x^2}{2}}dx
\\
&\ge & \sqrt{\frac{\tau}{8}}e^{-\frac{1}{2}} =  \frac{\sqrt{\tau}}{\sqrt{8e}} > 1
\end{eqnarray*}
The lemma follows since $\tilde{g}$ is increasing. \proofbox

\begin{lemma}\label{lem:derivative_bound}
Suppose $\log\left(\frac{1}{Z}\right)\le \frac{\tau}{16},\;
Z\le\frac{1}{e}$ and $\tau\ge 8e$. For every segment $I\subset
\reals$ of length $\le 2$ and every $x\in I$ we have\footnote{when
$s$ is $0$ or $U_{\tau,Z}$ (i.e, when $g$ is not differentiable),
$|g'(s)|$ stands for the maximum of the absolute values of the left
and right derivatives. }
\begin{equation}\label{eq:1}
4\max_{s\in I}|g'(s)|\le \frac{1}{\tau}xg(x)+Z
\end{equation}
\end{lemma}
\proof Let $I=[a,b]$. The function $\frac{1}{\tau}xg(x)+Z$ is non
decreasing (since it is constant outside $[0,U_{\tau,Z}]$ and is the
derivative of the convex function $8\tilde{g}$ inside
$[0,U_{\tau,Z}]$). Therefore, it is enough to show that
\begin{equation}\label{eq:1.1}
4\max_{s\in [a,b]}|g'(s)|\le \frac{1}{\tau}ag(a)+Z
\end{equation}
We first claim that we can restrict to the case that $I\subset
[0,U_{\tau,Z}]$. Indeed, if $a>U_{\tau,Z}$ or $b<0$, then the l.h.s.
is $0$ and the claim holds, so we can assume that $a\le U_{\tau,Z}$
and $b\ge 0$. Next, if we replace $b$ with $\min\{b,U_{\tau,Z}\}$,
both the l.h.s. and r.h.s. of (\ref{eq:1.1}) remains unchanged, as
$g'(s)=0$ for $s>U_{\tau,Z}$. Therefore, we can also assume that
$b\le U_{\tau,Z}$. Likewise, if we replace $a$ with $\max\{a,0\}$,
both the l.h.s. and r.h.s. of (\ref{eq:1.1}) remains unchanged, as
$g'(s)=0$ for $s<0$ and $\frac{1}{\tau}ag(a)+Z=Z$ for $a\le 0$.

By Gronwall's inequality for ODE, if $g''(x)\le \frac{1}{4}g'(x)$,
we will have that for all $s\in I$,
\begin{eqnarray*}
g'(s) &\le& g'(a)\exp\left(\frac{s-a}{4}\right)
\\
&\le& g'(a)\sqrt{e}
\\
&= &  \frac{\sqrt{e}}{8\tau}ag(a)+\frac{\sqrt{e}}{8}Z
\\
&\le &  \frac{1}{4\tau}ag(a)+\frac{1}{4}Z
\end{eqnarray*}
Therefore, it is sufficient to show that for all $x\in I$ we have
\begin{equation}\label{eq:2}
g''(x)\le \frac{1}{4}g'(x)~.
\end{equation}
By
(\ref{eq:scond_derivative}), since $I\subset [0,U_{\tau,Z}]$ and
since $\log\left(\frac{1}{Z}\right)\le \frac{\tau}{16}$, by
Lemma~\ref{lem:basic_g}, we have $U_{\tau,Z}\le
\sqrt{16\tau\log\left(\frac{1}{Z}\right)}\le \tau$. Since $x\leq
U_{\tau,Z} \leq \tau$ which implies,
\begin{eqnarray*}
g''(x) &\le& \frac{1}{8\tau}g(x) + \frac{x}{8\tau}g'(x)
\\
&\le& \frac{1}{8\tau}g(x) + \frac{1}{8}g'(x)
\end{eqnarray*}
It therefore remains to show that $\frac{1}{8\tau}g(x) \le  \frac{1}{8}g'(x)$. By (\ref{eq:diff_eq}) it is equivalent to
\[
\frac{g(x)}{\tau}\le \frac{xg(x)}{8\tau}+\frac{Z}{8}
\]

For $x \ge 8$ the inequality is clear. For $x\le 8$ we will show that $\frac{g(x)}{\tau}\le \frac{Z}{8}$. Indeed, for such $x$, since $\tau\ge 8e$, we have
\begin{eqnarray*}
g(x) &\le&  \sqrt{\frac{\tau}{8}}Z\left(\int_0^{\frac{8}{\sqrt{8\tau}}}\exp\left(-\frac{s^2}{2}\right)ds\right)\exp\left(\frac{4}{\tau}\right)
\\
&\le&  \sqrt{\frac{\tau}{8}}Z\frac{8}{\sqrt{8\tau}}\exp\left(\frac{4}{\tau}\right)
\\
&\le&  \sqrt{\frac{\tau}{8}}Z\frac{8}{\sqrt{8\cdot 8e}}\exp\left(\frac{4}{8e}\right)
\\
&\le&   \sqrt{\frac{\tau}{8}}Z\sqrt{e}\le \frac{\tau Z}{8}
\end{eqnarray*}
\proofbox

\paragraph{Removing the projection and restricting to a finite horizon}\label{sec:removing_proj}

The next step is to show that in order to prove Theorem
\ref{thm:two_experts}, we can consider a version of algorithm
\ref{alg:two_with_proj} with finite time horizon and no projection.
Concretely, consider the following algorithm:
\begin{algorithm}[H]
\caption{Two experts without projection and with bounded horizon} \label{alg:two_without_proj}
{\bf Parameters:} Initial $x_1\in [-2,U+2]$, $T,\tau,Z, D$.
\begin{algorithmic}[1]
\FOR {$t=1,2, \ldots, T$}
\IF{$D\log\left(\frac{1}{Z}\right) \le \frac{\tau}{64}$}
\STATE Predict $g(x_t)$
\ELSE
\STATE Predict $0$
\ENDIF
\STATE Let $b_t=\frac{l_t(0)-l_t(1)}{\sqrt{D}}$
\STATE Update $x_{t+1}=\left(1-\frac{1}{\tau}\right) x_t +  b_t$
\ENDFOR
\end{algorithmic}
\end{algorithm}
For expert $i\in \{0,1\}$ and an interval $I$ we denote by
$R_{I,\tau,D,Z}^{i,alg2}$ the worst case regret of algorithm
\ref{alg:two_with_proj} on the interval $I$ when running with
parameters $\tau,D,Z$. Likewise, we denote by
$R_{T,\tau,D,Z}^{i,alg3}$ the worst case regret (over all possible
loss sequences and initial points $x_1$) of algorithm
\ref{alg:two_without_proj} when running with parameters $\tau,D,Z,
T$.

\begin{lemma}\label{lem:reduction_to_finite_horizon}
$R_{I,\tau,D,Z}^{i,alg2} \le R_{|I|,\tau,D,Z}^{i,alg3}$
\end{lemma}
\proof
Denote $T=|I|$ and $I=\{K+1,\ldots,K+T\}$.
We claim that there exists a sequence $\cl=\{l_1,l_2,\ldots \} \subset [0,1]^2$ of losses such that
\begin{itemize}
    \item
Algorithm \ref{alg:two_with_proj} suffers a regret of
$R_{I,\tau,D,Z}^{i,alg2}$ on the segment $I$ when running on $\cl$
    \item
If we let
$b_1=\frac{l_1(0)-l_1(1)}{\sqrt{D}},b_2=\frac{l_2(0)-l_2(1)}{\sqrt{D}},\ldots$
and let $x_1,x_2,\ldots$ be the actions that algorithm
\ref{alg:two_with_proj} chose, then $\forall t\in I, \;\tilde
x_{t+1}:=\left(1-\frac{1}{\tau}\right)x_t+b_t\in [-2,U+2]$.
\end{itemize}
This will prove the lemma, because in that case the actions, and
therefore the regret of algorithm \ref{alg:two_without_proj} on the
sequence $l_{K+1},\ldots,l_{K+T}$ with initial point $x_{K+1}$ is
identical to algorithm \ref{alg:two_with_proj} on $I$. In
particular, $R_{I,\tau,D,Z}^{i,alg2} \le R_{|I|,\tau,D,Z}^{i,alg3}$

Assume toward a contradiction that there is no such $\cl$. Choose
$\cl$ among all sequences causing a regret of
$R_{I,\tau,D,Z}^{i,alg2}$, in a way that the first step $t\in I$ for
which $\tilde x_{t+1}\notin [-2,U+2]$ is as large as possible.
Assume that $\tilde x_{t+1}> U+2$ (a similar argument holds if
$\tilde x_{t+1} < -2$). This implies that $l_t(0) > l_t(1)$ and
$g(x_t)=1$. Now, suppose we generate a new sequence by decreasing
$l_t(0)$ in a way that we would have $\tilde x_{t+1}=U+2$. This will
not change the loss of the algorithm in step $t$ (as $g(x_t)=1$) and
won't change $x_{t+1}$, and therefore won't change the actions of
the algorithm and its losses in the remaining steps. As for the
experts, this will only improve the loss of expert $0$. Therefore,
the regret will not decrease. This contradicts the minimality of
$\cl$. \proofbox

\paragraph{Completing the proof}
We first discuss the case that $D\log\left(\frac{1}{Z}\right) >
\frac{\tau}{64}$, which is much simpler. In that case the algorithm
will simply choose the expert $0$ at each round. Hence, the regret
w.r.t. expert $0$ will be zero. Likewise, since the algorithm does
not move at all, the regret w.r.t. expert $1$ on an interval of
length $\tau$ is at most $\tau$. Since
\[
\tau = \sqrt{\tau}\sqrt{\tau} < \sqrt{\tau}\sqrt{64
D\log\left(\frac{1}{Z}\right)}
\]
we are done. Therefore, for the rest of the proof, we assume that
$D\log\left(\frac{1}{Z}\right) \le \frac{\tau}{64}$. In this case,
and by lemma \ref{lem:reduction_to_finite_horizon}, it is enough to
prove the following lemma:
\begin{lemma}
\label{lemma:alg3}
 Suppose $D\log\left(\frac{1}{Z}\right)\le
\frac{\tau}{64}$ and $Z\le\frac{1}{e}$. Then, for any initial point
$x_1\in [-2,U+2]$, algorithm \ref{alg:two_without_proj} guarantees
the following,
\begin{itemize}
\item
The regret  w.r.t. expert $0$ is at most
$\min\{\sqrt{D}TZ,\sqrt{16D\tau\log\left(\frac{1}{Z}\right)}+2\sqrt{D}+\sqrt{D}TZ\}$
\item
If $T\le \tau$, the regret w.r.t. expert $1$ is at most
$\sqrt{64D\tau\log\left(\frac{1}{Z}\right)}+\sqrt{D}4+\sqrt{D}\tau
Z$
\end{itemize}
\end{lemma}

The proof is almost identical to the one sketched in
Section~\ref{sec:proof-overview} and is deferred to the
Appendix.

\subsection{Combining Algorithms}

\subsubsection{Combining two algorithms}
We next describe a variant of algorithm~\ref{alg:two_with_proj}
to combine two algorithms $\ca_0$ and $\ca_1$ for linear
optimization over $\Delta(N)$ with switching costs $D$. The
resulting algorithm will have tiny regret w.r.t. to $\ca_0$ and
small regret w.r.t. $\ca_1$. To describe it and analyze it, we will
use the following terminology. We will say that and algorithm is
{\em $M$-slow} if for every loss sequence, the distance between two
consecutive actions is at most $M$.

\begin{algorithm}[H]
\caption{Two algorithm combiner (with parameters $\tau, Z, M$ and
$D$)} \label{alg:two_combiner} {\bf Parameters:} $\frac{M}{D}$-slow
algorithms $\ca_0, \ca_1$ for online linear optimization over
$\Delta(N)$ with switching costs $D$.
\begin{algorithmic}[1]
\STATE Set $x_t=0$, $g=g_{\tau,Z}$, $U=U_{\tau,Z}$ \FOR {$t=1,2,
\ldots$} \STATE Let $z^0_t, z_t^1$ be the actions of $\ca_0,\ca_1$
\IF {$\tau\ge 64D\log\left(\frac{1}{Z}\right)$} \STATE Predict
$g(x_t)z^0_t+(1-g(x_t))z^1_t$ \ELSE \STATE Predict $z^0_t$ \ENDIF
\STATE Obtain loss vector $l_t$ and let $\tilde l_t(i) =
\frac{\inner{l_t,z^i_t}+D\|z^i_t - z^i_{t+1}\|}{M+1}$ be the scaled
loss of algorithm $i$ \STATE Let $b_t=\frac{\tilde l_t(0)-\tilde
l_t(1)}{\sqrt{D}}$ \STATE Update $x_{t+1}=\Pi_{[-2,U+2]}\left[
\left(1-\frac{1}{\tau}\right) x_t + b_t\right]$ \ENDFOR
\end{algorithmic}
\end{algorithm}

Using Theorem \ref{thm:two_experts} we conclude that
\begin{theorem}\label{thm:two_combiner}
Suppose $Z\le\frac{1}{e}$. Algorithm \ref{alg:two_combiner}
guarantees:
\begin{itemize}
\item For every time interval $I$, the regret w.r.t.  $\ca_0$, i.e.,
$(M+1)\sum_{t\in I} \tilde l_t(0)$, is at most
$$\left((M+1)\sqrt{D}\right)\min\{TZ, \sqrt{16 \tau\log\left(\frac{1}{Z}\right)}+2+|I|Z\}$$
\item For every time interval $I$ of length $\le \tau$, the regret w.r.t. $\ca_1$, i.e.,
$(M+1)\sum_{t\in I} \tilde l_t(1)$, is at most
$$\left((M+1)\sqrt{D}\right)\left(\sqrt{64\tau\log\left(\frac{1}{Z}\right)}+4+\tau Z\right)$$
\end{itemize}
\end{theorem}
\proof  The loss of the combined algorithm at time $t$ is

\begin{eqnarray*}
    l_t^{\mathrm{comb}}&=& \inner{l_t,g(x_t)z_t^0+(1-g(x_t))z_t^1}\\&&+D\| g(x_t)z_t^0+(1-g(x_t))z_t^1-g(x_{t+1})z_{t+1}^0+(1-g(x_{t+1}))z_{t+1}^1\|
    \\
    &\le& \inner{l_t,g(x_t)z_t^0+(1-g(x_t))z_t^1}+D\| g(x_t)(z_t^0-z_{t+1}^0)+(1-g(x_t))(z_t^1-z_{t+1}^1)\|
    \\
    & &  +D\| (g(x_t)-g(x_{t+1}))z_{t+1}^0+((1-g(x_{t}))-(1-g(x_{t+1})))z_{t+1}^1\|
    \\
    &\le& \inner{l_t,g(x_t)z_t^0+(1-g(x_t))z_t^1}+g(x_t)D\| z_t^0-z_{t+1}^0\|+(1-g(x_t))D\|z_t^1-z_{t+1}^1\|
    \\&&+D|g(x_t)-g(x_{t+1})|
    \\
    &=& (M+1)\left(g(x_t)\tilde l_t(0)+(1-g(x_t))\tilde l_t(1)\right)+D|g(x_t)-g(x_{t+1})|
    \\
    &\le & (M+1)\left(g(x_t)\tilde l_t(0)+(1-g(x_t))\tilde l_t(1)+D|g(x_t)-g(x_{t+1})|\right)
\end{eqnarray*}
By Theorem \ref{thm:two_experts}, for every time interval $I$ we
have
\begin{eqnarray*}
\sum_{t\in I} g(x_t)\tilde l_t(0)+(1-g(x_t))\tilde
l_t(1)+D|g(x_t)-g(x_{t+1})|\le 
\\
\sum_{t\in I} \tilde l_t(0) +
\min\left\{\sqrt{D}|I|Z,\sqrt{16D\tau\log\left(\frac{1}{Z}\right)}+2\sqrt{D}+\sqrt{D}|I|Z\right\}
\end{eqnarray*}
Therefore we have that
\[
\sum_{t\in I} l_t^{\mathrm{comb}}- l_t^{\ca_0}\leq
\left((M+1)\sqrt{D}\right)\min\{TZ,
\sqrt{16\tau\log\left(\frac{1}{Z}\right)}+2+|I|Z\}
\]
Finally, if $|I|\le\tau$ then
\small
\[
\sum_{t\in I} g(x_t)\tilde l_t(0)+(1-g(x_t))\tilde
l_t(1)+D|g(x_t)-g(x_{t+1})|\le \sum_{t\in I} \tilde l_t(1) +
\sqrt{64D\tau\log\left(\frac{1}{Z}\right)}+4\sqrt{D}+\sqrt{D}\tau Z
\]
\normalsize
which implies that
\[
\sum_{t\in I} l_t^{\mathrm{comb}}- l_t^{\ca_1}\leq
\left((M+1)\sqrt{D}\right)\left(\sqrt{64\tau\log\left(\frac{1}{Z}\right)}+4+\tau
Z\right)
\]
\proofbox

We next bound the slowness of the combined algorithm.
\begin{lemma}\label{lem:combined_slow}
Suppose $Z\le \frac{1}{e}$. The combined algorithm is
$\left(\frac{M}{D}+\sqrt{\frac{\log\left(\frac{1}{Z}\right)}{4 \tau
D}} + \frac{Z}{8\sqrt{D}} \right)$-slow if $\tau\ge
64D\log\left(\frac{1}{Z}\right)$ and $\frac{M}{D}$-slow otherwise.
\end{lemma}

\begin{proof}
Clearly, if $\tau< 64D\log\left(\frac{1}{Z}\right)$, we select
$z^0_t$ and the bound follows from the bound on $\ca_0$.

Assume that $\tau\ge 64D\log\left(\frac{1}{Z}\right)$.
The movement of the combined algorithm at
each step is bounded by a convex combination of the movements of
$\ca_0$ and $\ca_1$ plus $|g(x_t)-g(x_{t+1})|$. More specifically
\[
g(x_{t+1})z_{t+1}^0 - g(x_{t})z_{t}^0
=g(x_{t+1})(z^0_{t+1}-z^0_t)+(g(x_{t+1})-g(x_t))z_t^0
\]
Similarly
\[
(1-g(x_{t+1}))z_{t+1}^1 - (1-g(x_{t}))z_{t}^1 =(z_{t+1}^1 -z_t^1)(1-
 g(x_{t+1}))+(g(x_{t})-g(x_{t+1}))z_t^1
\]
Since we have that algorithms $\ca_b$ are $M/D$-slow we have
$z_{t+1}^b -z_t^b\leq M/D$.
What remains is to bound
\[
|g(x_t)-g(x_{t+1})| \le |x_t-x_{t+1}|\max_\xi|g'(\xi)|\le
\frac{\max_\xi|g'(\xi)|}{\sqrt{D}},
\]
where we used the fact that $|x_t-x_{t+1}|\leq |b_t|\leq
1/\sqrt{D}$.
Since $g'$ is nondecreasing in $[0,U_{\tau,Z}]$ and is $0$ outside,
we have
\[
\max_\xi|g'(\xi)| = g'(U_{\tau,z}) = \frac{U_{\tau,z}g(U_{\tau,z})}{8\tau}+ \frac{Z}{8} = \frac{U_{\tau,z}}{8\tau}+ \frac{Z}{8}
\]
and the lemma follows, since by Lemma \ref{lem:basic_g} we have
$U_{\tau,Z}\leq \sqrt{16 \tau \ln(1/Z)}$.
\end{proof}

\subsubsection{Combining many algorithms}
For simplicity, let's assume that $T=2^{K}$. Let $\ca_{\base},
\ca_0,\ldots,\ca_{K-1}$ be $\frac{1}{D}$-slow algorithms. We next
explain how one can sequentially use algorithm
\ref{alg:two_combiner} to combine these algorithms into a single
algorithm $\ca$, while preserving some of their guarantees. Namely,
on every interval $I$ of length $|I|\in[2^{-u-1}T,\le 2^{-u}T]$,
$\ca$ will have small loss w.r.t. $\ca_u$. In addition, on the
entire segment $[T]$, $\ca$ will have essentially no loss w.r.t.
$\ca_\base$. We note that in our application, the role of the
$\ca_u$'s will be to ensure strong adaptivity. On the other hand,
the role of $\ca_\base$ will be to ensure competitive ratio, and it
can be easily omitted from when this is not needed.

We will build algorithms $\cb_{-1},\cb_0,\ldots,\cb_{K-1}$ where
$\cb_{-1}=\ca_{\base}$, and for $u\ge 0$ the algorithm $\cb_{u}$ is
obtained from $\cb_{u-1}$ by combining it with $\ca_{u}$ using
algorithm \ref{alg:two_combiner} (where $\cb_{u-1}$ plays the role
of $\ca_0$ from algorithm \ref{alg:two_combiner} and $\ca_u$ the
role of $\ca_1$). Finally, we will take $\ca=\cb_{K-1}$. The
parameter $Z$ will be the same among all the applications of
algorithm \ref{alg:two_combiner}, but we will assume that $Z\le
\frac{1}{e}$. The parameter $\tau$ will be set to $2^{-u}T$ when we
combine $\ca_u$. Lastly, the slowness bound on $\ca_u$ will be
$\frac{1}{D}$, while the slowness bound on $\cb_{u-1}$ is the one
implied by Lemma~\ref{lem:combined_slow}. Namely,
\[
\max\left(\frac{1}{D} +
\sum^{u-1}_{i=0}\sqrt{\frac{\log\left(\frac{1}{Z}\right)}{4
(2^{-i}T) D}} + \frac{Z}{8\sqrt{ D}}\;,\;1\right)
\]

\begin{theorem}\label{thm:strongly_adaptive}
    Assume $Z\le \frac{1}{e}$
    \begin{enumerate}
    \item On each interval $I$ of length $2^{-u-1}\le |I| \le 2^{-u}T$, the regret of $\ca$ w.r.t.\ $\ca_u$ is\\ $O\left(\sqrt{D|I|\log\left(\frac{1}{Z}\right)}+\sqrt{D}\log(|I| )|I|Z\right)$
    \item The regret of $\ca$ w.r.t. $\ca_\base$ is $2\sqrt{D}T\log(T)Z$
    \end{enumerate}
\end{theorem}
\proof (sketch) As with previous proofs, we can assume w.l.o.g. that
$T\ge D\log_2\left(T\right)$. We first claim that for every $u$,
$\cb_u$ is $\frac{2}{D}$-slow, and hence we can use
Theorem~\ref{thm:two_combiner} with $M=2$. Indeed, $\cb_0=\ca_0$ is
$\frac{1}{D}$-slow by assumption. Now, when we go from $\cb_{u-1}$
to $\cb_{u}$, by Lemma~\ref{lem:combined_slow}, the slowness grows
by
\[
\sqrt{\frac{\log\left(\frac{1}{Z}\right)}{4 (2^{-u}T) D}} +
\frac{Z}{8\sqrt{D}}
\]
as long as $2^{-u}T\ge 64 D\log\left(\frac{1}{Z}\right)$, and by $0$ after that. It follows that the total growth is bounded by
\small
\begin{eqnarray*}
    \sum_{u=0}^{\left\lfloor\log_2\left(\frac{T}{64D\log\left(\frac{1}{Z}\right)}\right)\right\rfloor} \sqrt{\frac{\log\left(\frac{1}{Z}\right)}{4 2^{-u}T D}} + \frac{Z}{8\sqrt{ D}}
    &\le& \sqrt{\frac{\log\left(\frac{1}{Z}\right)}{4 TD}}\sqrt{2}^{\log_2\left(\frac{T}{64 D\log\left(\frac{1}{Z}\right)}\right)}\frac{\sqrt{2}}{\sqrt{2}-1} + \frac{\log_2(T)Z}{8\sqrt{ D}}
    \\
    &=& \sqrt{\frac{\log\left(\frac{1}{Z}\right)}{4 T D}}
    \sqrt{\frac{T}{64D\log\left(\frac{1}{Z}\right)}}
    \frac{\sqrt{2}}{\sqrt{2}-1} + \frac{\log_2(T)}{8T\sqrt{ D}}
    \\
    &\le &
    \frac{1}{\sqrt{256}D}4 + \frac{1}{8 D} \le \frac{1}{2 D}
\end{eqnarray*}
\normalsize
Now, let $I$ be an interval of length $\le 2^{-u}T =:\tau$ for some
$u$. By Theorem \ref{thm:two_combiner}, the regret of $\cb_u$ w.r.t.
$\ca_u$ on that interval is, up to a universal multiplicative
constant, at most
\[
\sqrt{D\tau\log\left(\frac{1}{Z}\right)}+\sqrt{D}\tau Z
\]
Now, in order to go from $\cb_u$ to $\cb_{K}$, we sequentially
combine the algorithms $\ca_{u+1},\ldots,\ca_{K}$. By Theorem
\ref{thm:two_combiner}, up to a universal multiplicative constant,
this adds to the regret on the given segment at most
\small
\begin{eqnarray*}
    \sum_{r=1}^{K-u-1}\left(\sqrt{D\cdot 2^{-r} \tau\log\left(\frac{1}{Z}\right)}+\sqrt{D}\tau Z\right) &\le & \sqrt{D}\left(\sum_{r=1}^{\infty}2^{-\frac{r}{2}}\right)\sqrt{ \tau\log\left(\frac{1}{Z}\right)}+\sqrt{D}(K-u-1)\tau Z
    \\
    &\le & 20\sqrt{D\tau\log\left(\frac{1}{Z}\right)}+\sqrt{D}\log(\tau )\tau Z
\end{eqnarray*}
\normalsize
This proves the first part of the Theorem. The proof of the second part is similar.
\proofbox

\proof (of Theorem \ref{thm:main_combine}) The proof follows from
Theorem \ref{thm:strongly_adaptive} with $Z=\frac{1}{2T\log(T)}$.
Indeed, in the case that $D=1$, the requirement of being
$\frac{1}{D}$-slow always holds. The general case follows by a
simple scaling argument. As for running time, in the case that the
algorithms choose a specific expert (rather than a distribution on
the expert), in order to apply algorithm \ref{alg:two_combiner} all
is needed is the loss of the chosen expert and an indication weather
the algorithms made switches. \proofbox

\section{Experts and Metrical Tasks Systems}
\subsection{Algorithms}
In this section we prove Theorem \ref{thm:main_non_uniform}. We will
use Theorem \ref{thm:strongly_adaptive} where the basic algorithms
are the fixed share algorithm~\cite{herbster1998tracking}.
We extend its analysis to handle switching costs.

\begin{algorithm}[H]
    \caption{Fixed Share~\cite{herbster1998tracking}} \label{alg:MW1}
    {\bf Parameters:} $\tau, D$
    \begin{algorithmic}[1]
        \STATE Set $\eta=\sqrt{\frac{\log\left(N\tau\right)}{D\tau}}$
        \STATE Set $z_1=\left(\frac{1}{N},\ldots,\frac{1}{N}\right)$
        \FOR {$t=1,2, \ldots, $}
        \STATE Predict $z_t$
        \IF {$\tau\ge 16D\log\left(N\tau\right)$}
        \STATE Update $z_{t+1}(i)=\frac{z_t(i)e^{-\eta l_t(i)}+\frac{1}{N\tau}}{\sum_{j=1}^N\left(z_t(j)e^{-\eta l_t(j)}+\frac{1}{N\tau}\right)}$
        \ENDIF
        \ENDFOR
    \end{algorithmic}
\end{algorithm}

We first bound the rate of change in the action distribution which
bounds the slowness of the algorithm.
\begin{lemma}\label{lem:mul_update_sc}
    Let $\eta>0, \tau\ge \frac{2}{\eta}$ and $l_1,\ldots,l_N\in [0,1]$.
    Let $z\in \Delta(N)$ and define $z'(i)=\frac{e^{-\eta l_i}z(i)+\frac{1}{N\tau}}{\sum_{j=1}^N e^{-\eta l_j}z(j)+\frac{1}{N\tau}}$. Then $\|z-z'\|\le \eta$
\end{lemma}
\begin{proof}
Denote $\tilde z(i)=z(i)e^{-\eta l_i}+\frac{1}{N\tau}$. For the
proof it will be more convenient to use norm $L_1$ and recall that
$\|z\|_1= 2\|z\|$. We have
\begin{equation}\label{eq:3}
\|z'-z\|_1 \le \|z'-\tilde z\|_1 +\|\tilde z -z\|_1 ~.
\end{equation}
We bound the contribution of each term independently. For the second
term we have,
\begin{eqnarray}\label{eq:4}
\|\tilde z-z\|_1 &\le&\sum_{i=1}^N |z(i)(1-e^{-\eta
l_i})|+\frac{1}{N\tau}\nonumber
\\
&=&\frac{1}{\tau}+\sum_{i=1}^N z(i)(1-e^{-\eta l_i})\nonumber
\\
&\le&\frac{1}{\tau}+\sum_{i=1}^N z(i)\eta l_i\nonumber
\\
&\le&\frac{1}{\tau}+\sum_{i=1}^N z(i)\eta =\frac{1}{\tau}+\eta
\end{eqnarray}
For the first term we have,
\begin{eqnarray*}
    \|\tilde z-z'\|_1 &=&\left\|\tilde z-\frac{\tilde z}{\|\tilde
    z\|_1}\right\|_1
=\left|1-\frac{1}{\|\tilde z\|_1}\right|\cdot \|\tilde z\|_1
    =\left|\|\tilde z\|_1-1\right|
\end{eqnarray*}
To bound this we have,
\[
\|\tilde z  \|_1 \geq \frac{1}{\tau}+\| z\|_1 e^{-\eta} =
\frac{1}{\tau}+ e^{-\eta} \geq  \frac{1}{\tau} +1 - \eta
\]
and also
\[
\|\tilde z \|_1 \leq \|z\|_1 +\frac{1}{\tau}=1+\frac{1}{\tau}
\]
and we have
\[
\left|\|\tilde z\|_1-1\right| \leq \max\{\frac{1}{\tau},
\frac{1}{\tau}-\eta\} = \frac{1}{\tau}
\]
Combining with equations
(\ref{eq:3}) and (\ref{eq:4}), and since $\tau\ge \frac{2}{\eta}$, we conclude that
\[
\|z' -z\| = \frac{\|z' -z\|_1}{2}\le  \frac{1}{\tau} +\frac{\eta}{2} \le \eta
\]
\end{proof}

An immediate corollary is bounding the slowness of the algorithm.

\begin{corollary}\label{cor:basic_slow}
    Algorithm \ref{alg:MW1} is $\sqrt{\frac{\log\left(N\tau\right)}{D\tau}}$-slow when $\tau\ge 16D \log\left(N\tau\right)$ and $0$-slow otherwise.
\end{corollary}

We can now derive the regret bounds.

\begin{theorem}\label{thm:basic_alg_mw_1}
On every time interval of length $\le \tau$, the regret (including
switching costs) of $\mathrm{MW}^1$ is bounded by
$\sqrt{16D\tau\log\left(N\tau\right)}$
\end{theorem}

\begin{proof}
If $\tau< 16D\log\left(N\tau\right)$, the algorithm makes no moves,
so its regret is bounded by $\tau = \sqrt{\tau}\sqrt{\tau}\le
\sqrt{3D\tau}\le \sqrt{16D\tau\log\left(N\tau\right)}$. We can
therefore assume that $\tau\ge 16D\log\left(N\tau\right)$.
\cite{hazan2015computational} showed that in this case the regret
of the algorithm, excluding switching costs is at most
$\frac{2\log\left(N\tau\right)}{\eta}+\eta \tau$. By
Lemma~\ref{lem:mul_update_sc} (the fact that $\tau\ge
\frac{2}{\eta}$ follows from the assumption that $\tau\ge
16D\log\left(N\tau\right)$) the switching cost in each round is
bounded by $D\eta$. Hence, the regret is at most
$\frac{2\log\left(N\tau\right)}{\eta}+\eta\tau+D\eta \tau\le
\frac{2\log\left(N\tau\right)}{\eta}+2D\eta
\tau=\sqrt{16D\log\left(N\tau\right)\tau}$.
\end{proof}

We are now ready to prove Theorem \ref{thm:main_non_uniform}

\proof (of Theorem \ref{thm:main_non_uniform}) Let
$\ca_0,\ldots,\ca_{K-1}$ be instances of algorithm \ref{alg:MW1}
with parameters $\tau=2^{-0}T,\tau=2^{-1}T,\ldots,\tau=2^{-K+1}T$.
Let $\ca$ be the algorithm obtained by Theorem
\ref{thm:strongly_adaptive} with $Z=\frac{1}{2T\log(T)}$. By
Theorems \ref{thm:strongly_adaptive} and \ref{thm:basic_alg_mw_1} we
have that the regret of $\ca$ on every interval $I$ is
\[
O\left(\sqrt{D |I|\log(NT)} + \sqrt{D|I|\log(|I|)}  + \sqrt{D}\right) = O\left(\sqrt{D |I|\log(NT)}\right)
\]
\proofbox

\subsection{A lower bound}
In this section we will prove Theorem \ref{thm:main_lower} that
shows that the bound in Theorem \ref{thm:main_non_uniform} is
optimal up to a constant factor. We start by showing how the
adversary can generate sequences of guaranteed high loss.

\begin{lemma}\label{lem:loss_bound}
Let $\ca$ be an algorithm for linear optimization over $\Delta(N)$
for $N=2$. Suppose $\ca$ have a regret bound of $M$ on the interval
$[T]$ and assume that $T\ge 4M$. There is a sequence of loses
$l_1,\ldots,l_T\in
\left\{(0,1),(1,0),\left(\frac{1}{2},\frac{1}{2}\right)\right\}$ for
which the loss of the algorithm is at least $\frac{1}{2}T+M2^{-4M}$
\end{lemma}
\begin{proof}
Since we have two actions, our distribution over actions at time $t$
would be $(1-z_t,z_t)$.
Assume that the initial choice of the algorithm is
$z_1\le\frac{1}{2}$ (a similar argument holds when $z_1\ge
\frac{1}{2}$). Let $a=M2^{-4M}$ and assume toward a contradiction
that the is no such sequence. Namely, $\ca$ is guaranteed to have a
loss of at most $\frac{1}{2}T+a$ on every sequence. We denote by
$L_t$ the loss of the algorithm at the end of round $t$, and define
the {\em gain} of the algorithm at time $t$ as
$G_t=\frac{t}{2}-L_t$. We claim that we have $z_{t+1}\le
\frac{1}{2}+G_t+a$. Indeed, otherwise, the adversary can cause the
gain to be $<-a$ at the next step $t+1$ by choosing the loss vector
$(0,1)$. It can also keep the gain $<-a$ by repeatedly choosing the
loss vectors $\left(\frac{1}{2},\frac{1}{2}\right)$.

Consider now the action of the algorithm when the loss vectors are
always $(1,0)$. We claim that $z_t\le \frac{1}{2}+2^{t-1}a$. We will
prove this by induction. For $t=1$ it follows from our assumption
that $z_1\le \frac{1}{2}$. Assume that this is the case for all
$t'<t$. We have that the loss of the algorithm before the step $t$,
is at least
\begin{equation}\label{eq:5}
\left(\sum_{t'=1}^{t-1}\frac{1}{2}-2^{t'-1}a\right)=\frac{t-1}{2}-(2^{t-1}-1)a
\end{equation}
Therefore, $G_{t-1}\le (2^{t-1}-1)a$. It follows that
\[
z_t\le \frac{1}{2}+G_{t-1}+a\le \frac{1}{2}+(2^{t-1}-1)a+a=\frac{1}{2}+2^{t-1}a
\]
Now, using equation (\ref{eq:5}) again, the regret at time $t$
w.r.t. to the second expert is at least $\frac{t}{2}-2^ta\le M$.
Taking $t=4M$, it follows that $a\ge M2^{-4M}$
\end{proof}

We will now use the sequences guaranteed by the above lemma to show
a lower bound on the regret.

\begin{theorem}\label{thm:lower}
For every algorithm for online linear optimization over $\Delta(N)$
that runs for $T$ iterations, there is a segment $I$ on which the
regret is $\Omega\left(\sqrt{|I|\log\left(NT\right)}\right)$.
\end{theorem}

\begin{proof}
The known regret lower bounds guarantee that the regret is
$\Omega\left(\sqrt{|I|\log\left(N\right)}\right)$.
We first note that it is enough to show that that for some interval
$I$ the regret is $\Omega\left(\sqrt{|I|\log\left(T\right)}\right)$.
Hence, we will
have a regret lower bound of
\begin{eqnarray*}
\Omega\left(\max\left\{\sqrt{|I|\log\left(N\right)},\sqrt{|I|\log\left(T\right)}\right\}\right)
&= &
\Omega\left(\sqrt{|I|\log\left(N\right)}+\sqrt{|I|\log\left(T\right)}\right)
\\
&=&
\Omega\left(\sqrt{|I|\log\left(N\right)+|I|\log\left(T\right)}\right)
\\
&=& \Omega \left(\sqrt{|I|\log\left(NT\right)}\right)
\end{eqnarray*}
It is now left to show that for some interval $I$ the regret is
$\Omega\left(\sqrt{|I|\log\left(T\right)}\right)$. We will show that
this lower bound holds already in easier problem of linear
optimization over $\Delta(N)$, where $N=2$. Indeed, suppose toward a
contradiction that there is an algorithm $\ca$ whose regret is $\le
\frac{1}{100}\sqrt{|I|\log_2(T)}$ on every interval
$I\subset\{1,\ldots,T\}$.

Partition $[T]$ into intervals of size $\log_2(T)$. By
Lemma~\ref{lem:loss_bound}, there is a sequence of losses that all
come from the set
$\left\{(0,1),(1,0),\left(\frac{1}{2},\frac{1}{2}\right)\right\}$,
such that the loss on every interval is at least
$\frac{\log_2(T)}{2} +M2^{-4M}$ for $M=\frac{\log_2(T)}{100}$. The
loss over the entire interval $[T]$ is therefore at least
$\frac{T}{2} +\frac{T}{\log_2(T)}
M2^{-4M}=\frac{T}{2}+\frac{T}{100}2^{-\frac{\log_2(T)}{25}}=\frac{T}{2}+\frac{T^{\frac{24}{25}}}{100}$.
The regret is therefore at least $\frac{T^{\frac{24}{25}}}{100}$,
contradicting the assumption that it is at most
$\frac{1}{100}\sqrt{T\log(T)}$.
\end{proof}

\section{$k$-Server}\label{sec:k_server}
Recall that in the $k$-median problem we are given a metric space
$(X,d)$, and the goal is to find a set $A\subset X$ of size $k$ that
minimizes $\val_d(A) = \sum_{x\in X}d(x,A)$, where
$d(x,A)=\min_{a\in A} d(x,a)$.

\begin{lemma}
Assume that there is an efficient\footnote{Namely, one that runs in
each step in time polynomial in $n,T$ and the bit-representation of
the underlying metric-space.} algorithm for the $k$-server problem
with regret bound of $\poly(n)T^{1-\mu}$ for some $\mu>0$. Then,
for every $\epsilon>0$ there is an efficient
$(2+\epsilon)$-approximation algorithm for the $k$-median problem.
\end{lemma}
\begin{proof}
Let $\ca$ be an efficient algorithm for $k$-server with regret bound
of $f(n)T^{1-\mu}$ for some polynomially bounded function $f$. Let
$(X,d)$ be an $n$-points metric space that is an instance for the
$k$-median problem. Let $A^*\subset X$ be a set of $k$ points that
minimizes $\val_d$ and denote $\opt = \opt_d = \val_d(A^*)$.

We first note that we can assume w.l.o.g. that $\opt\ge 1$ and that
$D:=\max_{x,y}d(x,y)\le 3$. Indeed, we can compute a number $\opt\le
\alpha\le 3\opt$ using known approximation algorithms (e.g.,
using~\cite{byrka2015improved}). Now, instead of working with the
original metric $d$, we can work with the metric
$d'(x,x')=\frac{3\min(d(x,x'),\alpha)}{\alpha}$. Clearly, its
diameter is bounded by $3$. Moreover, we claim that $\opt_{d'}=
\frac{3\opt_d}{\alpha}$ and therefore $\opt_{d'}\ge 1$ and any
$(2+\epsilon)$-approximation w.r.t. $d'$ is also a
$(2+\epsilon)$-approximation w.r.t $d$. Indeed, since $d'\le
\frac{3}{\alpha}d$ we have $\opt_{d'}\le \frac{3\opt_d}{\alpha}$. On
the other hand, we claim that it cannot be the case that
$\val_{d'}(A)<\frac{3\opt_d}{\alpha}$. Indeed, in that case we must
have $d'(x,A)<\frac{3\opt_d}{\alpha}\le 3$ for all $x\in X$ in which
case $d'(x,A)= \frac{3}{\alpha}d(x,A)$. Hence, $\val_{d'}(A) =
\frac{3}{\alpha}\val_d(A)\ge \frac{3\opt_d}{\alpha}$. A
contradiction.

Suppose now that we run the $k$-server algorithm such that at each
round we choose a point $x\in X$ uniformly at random. If we run the
algorithm for
$T=\left(\frac{f(n)n}{\epsilon}\right)^{\frac{1}{\mu}}$ rounds, we
are guaranteed to have expected regret $\le \frac{\epsilon}{n} T$.
Denote by $A_t$ the location of servers at the beginning of round
$t$. The expected cost is at least $\sum_{t=1}^T \ymE_{x\sim
X}d(x,A_t)$. On the other hand, if $A^*\subset X$ is an optimal
solution to the $k$-medians problem, the expected loss of the
corresponding $k$-server strategy is $\sum_{t=1}^T \ymE_{x\sim
X}2d(x,A^*)$. Since the regret is bounded by $\frac{\epsilon}{n} T$,
we have $\sum_{t=1}^T \ymE_{x\sim X}d(x,A_t)-2d(x,A^*)\le
\frac{\epsilon}{n} T$. In particular, if we choose at random one of
the $A_t$'s as a solution to the $k$-median problem, we get a
solution with expected cost at most $\ymE_{x\sim
X}2d(x,A^*)+\epsilon = 2\opt+\epsilon \le (2+\epsilon)\opt$.
\end{proof}

\section{Paging}\label{sec:paging}
We first recall the paging problem. To simplify the presentation a
bit, we consider a version where both the losses and the movements
cost are divided by $2$. At each step $t$ the player has to choose a
set $A_t \in \binom{[N]}{k}$. Then, nature chooses an element
$i_t\in [N]$, and the player loses $1$ if $i_t\notin A_t$. In
addition, the player suffers a switching cost of $\frac{k}{2}$.
Consider first the multiplicative weight
algorithm~\cite{LittlestoneWa94} for the $N$-expert problem

\begin{algorithm}[H]
    \caption{MW} \label{alg:MW}
    {\bf Parameters:} $T, D$
    \begin{algorithmic}[1]
        \STATE Set $\eta=\sqrt{\frac{\log\left(N\right)}{2D T}}$
        \STATE Set $z_1=\left(\frac{1}{N},\ldots,\frac{1}{N}\right)\in \Delta(N)$
        \FOR {$t=1,2, \ldots, T$}
        \STATE Predict $z_t$
        \STATE Update $z_{t+1}(i)=\frac{z_t(i)e^{-\eta l_t(i)}}{\sum_{j=1}^N z_t(j)e^{-\eta l_t(j)} }$
        \ENDFOR
    \end{algorithmic}
\end{algorithm}

\begin{theorem}\label{thm:basic_alg_mw}
    The regret (including switching costs) of $\mathrm{MW}$ is bounded by $\sqrt{8D T\log\left(N\right)}$
\end{theorem}
\proof It is known (e.g., \cite{CesaBianchiLu06}, page 15) the
regret of the algorithm, excluding switching costs is at most
$\frac{\log\left(N\right)}{\eta}+\eta T$. By
Lemma~\ref{lem:mul_update_sc} (and taking $\tau$ to $\infty$) the
switching cost in each round is bounded by $\eta  D$. Hence, the
regret is at most $\frac{\log\left(N\right)}{\eta}+2D\eta T =
\sqrt{8 D\log\left(N\right)T}$. \proofbox

By Theorem \ref{thm:basic_alg_mw}, the MW algorithm has a regret of \mbox{$\sqrt{8\frac{k}{2}T\log
\binom{N}{k}} \le k\sqrt{4T\log (N)}$} in the paging problem.
However, a naive implementation of the algorithm will result with an
algorithm whose running time is exponential in $k$. As we explain
next, a more careful implementation will result with an efficient
algorithm. Let $\vec{\alpha}=(\alpha_1,\ldots,\alpha_N)\in
(0,\infty)^N$. For $A\subset [N]$ denote $\pi_\valpha(A) =
\prod_{i\in A}\alpha_i$ and $\Psi^\valpha_{N,k} = \sum_{A'\in
\binom{[N]}{k}} \pi_\valpha(A)$. Let $p^k_{\valpha}$ be the
distribution function on $\binom{[N]}{k}$ defined by
$p^k_{\vec{\alpha}}(A) = \frac{\pi_\valpha(A)}{\Psi^\valpha_{N,k}}$.
The multiplicative weights algorithm for paging can be described as
follows:
\begin{algorithm}[H]
    \caption{Multiplicative weights for paging prediction} \label{alg:k_out_of_N}
    \begin{algorithmic}[1]
        \STATE Set $\eta = \sqrt{\frac{\log\binom{N}{k}}{kT}}$
        \STATE Set $\valpha^1 = (1,\ldots,1)\in \reals^N$
        \FOR {$t=1,2, \ldots, T$}
        \STATE Choose a set $A_t\sim p^k_{\valpha^t}$ such that $\Pr(A_t\ne A_{t-1}) = \|p^k_{\valpha^t}-p^k_{\valpha^{t-1}}\|$ (when $t>1$)
        \STATE Update $\alpha^{t+1}_{i_t} = e^\eta\alpha^{t}_{i_t}$ and $\alpha_i^{t+1} = \alpha_i^{t}$ for all $i\ne i_t$.
        \ENDFOR
    \end{algorithmic}
\end{algorithm}
We next remark on the computational complexity of some sampling and
calculation procedures related to the family of distributions
$\{p^k_\valpha\}_{\valpha}$. The last point shows that algorithm
\ref{alg:k_out_of_N} can be implemented efficiently, which implies
Theorem \ref{thm:paing_non_adaptive_intro}. In the sequel, efficient
means polynomial in the the description length of $\valpha$. We denote $\valpha_- = (\alpha_1,\ldots,\alpha_{N-1})$

\begin{itemize}
    \item {\bf Computing.} In order to efficiently compute $p^k_\valpha(A)$ it is enough to efficiently compute $\Psi^\valpha_{N,k}$. This computation is straight forward when $k=1$ or $k=N$. When $1<k<N$ we have
    \begin{eqnarray*}
    \Psi^\valpha_{N,k} &=& \sum_{A\in \binom{[N]}{k},\; N\in A} \pi_\valpha(A) + \sum_{A\in \binom{[N]}{k},\; N\notin A} \pi_\valpha(A)
    \\
    &=& \alpha_N\sum_{A\in \binom{[N-1]}{k-1}} \pi_\valpha(A) + \sum_{A\in \binom{[N-1]}{k}} \pi_\valpha(A)
    \\
    &=& \alpha_N\Psi^{\valpha_-}_{N-1,k-1} + \Psi^{\valpha_{-}}_{N-1,k}
    \end{eqnarray*}
    hence, $\Psi^\valpha_{N,k}$ can be efficiently calculated using dynamic programming.

    \item {\bf Sampling.} Suppose that $A\in \binom{[N]}{k}$ is sampled form $p^k_\valpha$. We have
    \begin{equation}\label{eq:6}
    \Pr(N\notin A) = \frac{\Psi^{\valpha_-}_{N-1,k}}{\Psi^{\valpha}_{N,k}}\;\;\;\;\;\;\;\;\;
    \Pr(N\in A) = \frac{\alpha_N\Psi^{\valpha_-}_{N-1,k-1}}{\Psi^{\valpha}_{N,k}}
    \end{equation}
    Also for $A'\in \binom{[N-1]}{k-1}$ we have
    \[
    \Pr(A\setminus\{N\} = A'\mid N\in A) = \frac{p^k_\valpha(A'\cup\{N\})}{\Pr(N\in A)} = \frac{\frac{\alpha_N\pi_{\valpha_-}(A')}{\Psi^{\valpha}_{N,k}}}{\frac{\alpha_N\Psi^{\valpha_-}_{N-1,k-1}}{\Psi^{\valpha}_{N,k}}} = \frac{\pi_{\valpha_-}(A')}{\Psi^{\valpha_-}_{N-1,k-1}} = p_{\valpha_-}^{k-1}(A') ~.
    \]
    Likewise, for $A'\in \binom{[N-1]}{k}$ we have
    \[
    \Pr(A\setminus \{N\} = A'\mid N\notin A) = \frac{p^k_\valpha(A')}{\Pr(N\notin A)} = \frac{\frac{\pi_\valpha(A')}{\Psi^\valpha_{N,k}}}{\frac{\Psi^{\valpha_-}_{N-1,k}}{\Psi^{\valpha}_{N,k}}}  = \frac{\pi_\valpha(A')}{\Psi^{\valpha_-}_{N-1,k}} = p_{\valpha_-}^{k}(A')~.
    \]
    Hence, in order to efficiently sample from $p_\valpha^k$ we can first choose whether to include $N$ in $A$ according to equation (\ref{eq:6}). Then (recursively) sample $A\setminus \{N\}$ from
    $p_{\valpha_-}^{k-1}$ in the case that
    $N\in A$ and from $p_{\valpha_-}^{k}$ if $N\notin A$.
    \item {\bf Transitioning with minimal switching costs.} Let $\valpha\in (0,\infty)^N$ and $\valpha' = (\alpha_1,\ldots,\alpha_{N-1},\alpha_N + \delta)$. Suppose that $A\sim p_\valpha^k$ and we want to efficiently generate $A'\sim p^k_{\valpha'}$ such that $\Pr(A\ne A')=\|p^k_{\valpha}-p^k_{\valpha'}\|$. According to remark \ref{rem:sampling_min_cost}, in order to do that, it is enough to (1) efficiently compute $p^k_{\valpha}(A),p^k_{\valpha'}(A)$ which we already explained how to do,
    and to (2) efficiently sample from the distribution $q = \frac{(p^k_{\valpha'}-p^k_{\valpha})_+}{\|p^k_{\valpha'}-p^k_{\valpha}\|}$. We note that $q(A) > 0$ if and only if $N\in A$ in which case
    \begin{eqnarray*}
    q(A) &=& \frac{p^k_{\valpha'}(A)-p^k_{\valpha}(A)}{\|p^k_{\valpha'}-p^k_{\valpha}\|}
    \\
    &=& \frac{\pi_{\valpha'}(A)}{\|p^k_{\valpha'}-p^k_{\valpha}\|\Psi^{\valpha'}_{N,k}} - \frac{\pi_{\valpha}(A)}{\|p^k_{\valpha'}-p^k_{\valpha}\|\Psi^{\valpha}_{N,k}}
    \\
    &=& \left(\frac{\alpha_N+\delta}{\|p^k_{\valpha'}-p^k_{\valpha}\|\Psi^{\valpha'}_{N,k}} - \frac{\alpha_N}{\|p^k_{\valpha'}-p^k_{\valpha}\|\Psi^{\valpha}_{N,k}}\right)\pi_{\valpha_-}(A\setminus\{N\})~
    \end{eqnarray*}
    Since $\left(\frac{\alpha_N+\delta}{\|p^k_{\valpha'}-p^k_{\valpha}\|\Psi^{\valpha'}_{N,k}} - \frac{\alpha_N}{\|p^k_{\valpha'}-p^k_{\valpha}\|\Psi^{\valpha}_{N,k}}\right)$ does not depend in $A$, it follows that if $A\sim q$ then $A\setminus \{N\}\sim p_{\valpha_-}^{k-1}$.
    Hence, in order to sample from $q$, we can sample a set $\tilde A \in \binom{[N-1]}{k-1}$ according to $p_{\valpha_-}^{k-1}$ and then produce the set $A=\tilde A\cup \{N\}$
\end{itemize}

\section{Proof of Lemma~\ref{lemma:alg3}}

\proof
First, we note that since $D\ge 1$,
$D\log\left(\frac{1}{Z}\right)\le \frac{\tau}{64}$ and
$Z\le\frac{1}{e}$ we have
\[
\tau\ge 64D\log\left(\frac{1}{Z}\right) \ge 64 \ge 8e
\]
Hence, the assumptions in lemmas \ref{lem:basic_g} and
\ref{lem:derivative_bound} are satisfied. Let $G(x)=\int_0^xg(s)ds$.
Note that since $0\le g(x)\le 1$ for all $x$ and $g(x)=0$ for all
$x\le 0$ we have,
\begin{equation}\label{eq:5.5}
G(x)\le x_+
\end{equation}
Denote $\Phi_t=G(x_t)$ and let $I:\reals\to \{0,1\}$ be the
indicator function of the segment $[-2,U+2]$. By lemma
\ref{lem:basic_g} and the assumption that
$D\log\left(\frac{1}{Z}\right)\le \frac{\tau}{64}$ we have
\[
U + 2\le \sqrt{16\tau\log\left(\frac{1}{Z}\right)}+2 \le
\sqrt{16\tau\frac{\tau}{64D}}+2 \le \frac{\tau}{2\sqrt{D}} + 2 \le
\frac{\tau}{\sqrt{D}}
\]
In particular, $|x_1|\le \frac{\tau}{\sqrt{D}}$. Now, by induction
we have $|x_t|\le \frac{\tau}{\sqrt{D}}$  for every $t$. Indeed, if
$|x_t|\le \frac{\tau}{\sqrt{D}}$ then
\[
|x_{t+1}| = \left|\left(1-\frac{1}{\tau}\right)x_t + b_t\right| \le
\left(1-\frac{1}{\tau}\right)\frac{\tau}{\sqrt{D}}+\frac{1}{\sqrt{D}}=\frac{\tau}{\sqrt{D}}~.
\]
It follows that $\left|-\frac{x_t}{\tau}+b_t\right|\le
\frac{2}{\sqrt{D}}$. Now, we have
\small
\begin{eqnarray*}
\Phi_{t+1}-\Phi_{t}&=&\int_{x_t}^{x_t-\frac{1}{\tau} x_t+b_t}g(s)ds
\\
&\le & g(x_t)\left(-\frac{1}{\tau} x_t+
b_t\right)+\frac{1}{2}\left(-\frac{1}{\tau} x_t+
b_t\right)^2\max_{s\in [x_t,x_t-\tau^{-1}x_t+ b_t]}|g'(s)|
\\
&\le & g(x_t)\left(-\frac{1}{\tau} x_t+
b_t\right)+\frac{1}{2}\frac{4}{D}\max_{s\in [x_t,x_t-\tau^{-1}x_t+
b_t]}|g'(s)|
\\
&\le & g(x_t)\left(-\frac{1}{\tau} x_t+ b_t\right)+2\max_{s\in
[x_t,x_t-\tau^{-1}x_t+ b_t]}|g'(s)|
\\
&= & g(x_t)\left(-\frac{1}{\tau} x_t+ b_t\right)-2\max_{s\in
[x_t,x_t-\tau^{-1}x_t+ b_t]}|g'(s)| + 4\max_{s\in
[x_t,x_t-\tau^{-1}x_t+ b_t]}|g'(s)|
\\
&\le & g(x_t)\left(-\frac{1}{\tau} x_t+ b_t\right)- 2
\frac{1}{\left|-\frac{1}{\tau} x_t+ b_t\right|}
|g(x_t)-g(x_{t+1})|+4\max_{s\in [x_t,x_t-\tau^{-1}x_t+ b_t]}|g'(s)|
\\
&\le & g(x_t)\left(-\frac{1}{\tau} x_t+ b_t\right)-
\sqrt{D}|g(x_t)-g(x_{t+1})|+4\max_{s\in [x_t,x_t-\tau^{-1}x_t+
b_t]}|g'(s)|
\\
&\le & g(x_t)\left(-\frac{1}{\tau} x_t+ b_t\right)-
\sqrt{D}|g(x_t)-g(x_{t+1})|+\frac{1}{\tau}x_tg(x_t)I(x_t)+Z
\end{eqnarray*}
\normalsize
Here, the first inequality follows from the fact that for every
piece-wise differential function $f:[a,b]\to\reals$ we have
$\int_a^bf(x)dx\le f(a) + \frac{1}{2}(b-a)^2\max_{\xi\in
[a,b]}|f'(\xi)|$. The last inequality follows form lemma
\ref{lem:derivative_bound}. Summing from $t=1$ to $t=T$ we get
\begin{eqnarray*}
\Phi_{T+1}-\Phi_{1}\le \sum_{t=1}^Tg(x_t)
b_t-\sqrt{D}|g(x_t)-g(x_{t+1})|+\sum_{t=1}^T\frac{1}{\tau}x_tg(x_t)(I(x_t)-1)+TZ
\end{eqnarray*}
and after rearranging,
\[
-\sum_{t=1}^Tg(x_t) b_t+\sqrt{D}|g(x_t)-g(x_{t+1})| \le \Phi_{1} -
\Phi_{T+1} +\sum_{t=1}^T\frac{1}{\tau}x_tg(x_t)(I(x_t)-1)+TZ
\]
Hence, if we denote the loss of the algorithm by $L_T$, and by
$L_T^i$ the loss of the expert $i$, we have
\begin{eqnarray}\label{eq:0}
L_T &=&
\sum_{t=1}^Tg(x_t)l_t(1)+\sum_{t=1}^T(1-g(x_t))l_t(0)+D|g(x_t)-g(x_{t+1})|
\nonumber
\\
&=& L_T^0 + \sqrt{D}\left(-\sum_{t=1}^Tg(x_t)b_t +
\sqrt{D}|g(x_t)-g(x_{t+1})|\right) \nonumber
\\
&\le&
L_T^0+\sqrt{D}\Phi_{1}-\sqrt{D}\Phi_{T+1}+\sqrt{D}\sum_{t=1}^T\frac{1}{\tau}x_tg(x_t)(I(x_t)-1)+\sqrt{D}TZ
\end{eqnarray}
In particular, since $\Phi_{T+1}\ge 0$ and $x_tg(x_t)(I(x_t)-1)\le
0$ we have
\begin{eqnarray*}
L_T &\le& L_T^0+\sqrt{D}\Phi_{1} + \sqrt{D}TZ
\\
&=& L_T^0+\sqrt{D}G(x_1) + \sqrt{D}TZ
\\
&\le& L_T^0+\sqrt{D}x_1 + \sqrt{D}TZ
\\
&\le& L_T^0+\sqrt{D}U+\sqrt{D}2 + \sqrt{D}TZ
\\
&\le& L_T^0+\sqrt{D
16\tau\log\left(\frac{1}{Z}\right)}+\sqrt{D}2+\sqrt{D}TZ
\end{eqnarray*}
This proves the first and last parts of the lemma (as $G(x_1) = 0$
when $x_1=0$). It remains to bound $L_T - L_T^1$ when $T\le \tau$.
It is not hard to see that the worst case regret grows as the number
of rounds ($T$) grows but the other parameters remains the same.
Hence, we can assume w.l.o.g. that $T=\tau$. Now, note that
\[
\Phi_{T+1}=G(x_{T+1})=\int_{0}^{x_{T+1}}g(s)ds\ge x_{T+1}-U
\]
Also,
\[
x_tg(x_t)(1-I(x_t)) \ge x_t - U - 2~.
\]
Indeed, if $x_t< U$ than the r.h.s. is negative, while the l.h.s. is
always non-negative. If $x_t\ge U$ then the l.h.s. equals $x_t$ in
which case the inequality is clear. Therefore, if we denote
$b_0:=x_1$, we have
\footnotesize
\begin{eqnarray*}
\Phi_{T+1}+\sum_{t=1}^T\frac{1}{T}x_tg(x_t)(1-I(x_t)) &\ge &
(x_{T+1}-U) +\sum_{t=1}^T\frac{1}{T}(x_t-U-2)
\\
&=&
\sum_{j=0}^{T}\left(1-\frac{1}{T}\right)^{T-j}b_j+\sum_{t=1}^T\frac{1}{T}\sum_{j=0}^{t-1}\left(1-\frac{1}{T}\right)^{t-1-j}b_j-
2(U+1)
\\
&=&
\sum_{j=0}^{T}\left(1-\frac{1}{T}\right)^{T-j}b_j+\frac{1}{T}\sum_{j=0}^{T-1}\sum_{t=j+1}^T\left(1-\frac{1}{T}\right)^{t-1-j}b_j-
2(U+1)
\\
&=&
\sum_{j=0}^{T}\left(1-\frac{1}{T}\right)^{T-j}b_j+\frac{1}{T}\sum_{j=0}^{T}\frac{1-\left(1-\frac{1}{T}\right)^{T-j}}{\frac{1}{T}}b_j-
2(U+1)
\\
&=& \sum_{j=0}^{T}b_j- 2(U+1)
\\
&=& \frac{L_T^0-L_T^1}{\sqrt{D}} + x_1- 2(U+1)
\end{eqnarray*}
\normalsize
By equation (\ref{eq:0}) it follows that
\begin{eqnarray*}
L_T &\le& L_T^1+\sqrt{D}G(x_1)-\sqrt{D}x_1+\sqrt{D}2U + \sqrt{D}TZ +
2\sqrt{D}
\\
&\le& L_T^1+\sqrt{D}((x_1)_+-x_1)+\sqrt{D}2U + \sqrt{D}TZ +
2\sqrt{D}
\\
&=& L_T^1+\sqrt{D}(-x_1)_++\sqrt{D}2U + \sqrt{D}TZ + 2\sqrt{D}
\\
&\le& L_T^1+\sqrt{D}4+\sqrt{64DT\log\left(\frac{1}{Z}\right)}+
\sqrt{D}T Z
\end{eqnarray*}

\proofbox

\paragraph{Acknowledgements:}
We thank Eyal Gofer, Alon Gonen and Ohad Shamir for valuable discussions.

\bibliography{bib}

\end{document}